\documentclass[letterpaper]{article}

\usepackage[utf8]{inputenc}
\usepackage{amsmath,amscd,amssymb,amsthm}
\usepackage{fullpage}
\usepackage{natbib}
\usepackage{hyperref}

\usepackage{microtype}
\usepackage{graphicx}
\usepackage{subfigure}
\usepackage{booktabs} 

\usepackage{amsmath,amscd,amssymb,amsthm}
\usepackage{verbatim}
\usepackage{thmtools}
\usepackage{thm-restate}
\usepackage{algorithm}
\usepackage{algorithmic}

\declaretheorem[name=Theorem]{Theorem}

\declaretheorem[name=Lemma, numberlike=Theorem]{Lemma}

\newcommand{\R}{\mathbb{R}}
\newcommand{\E}{\mathop{\mathbb{E}}}

\newcommand{\bw}{\vec{w}}
\newcommand{\bx}{\vec{x}}

\newcommand{\by}{\vec{y}}

\usepackage{hyperref}

\title{Better SGD using Second-order Momentum}
\author{Hoang Tran\\Boston University\\\texttt{tranhp@bu.edu} \and Ashok Cutkosky\\Boston University\\\texttt{ashok@cutkosky.com}}
\date{}
\begin{document}

\maketitle
\begin{abstract}
 We develop a new algorithm for non-convex stochastic optimization that finds an $\epsilon$-critical point in the optimal $O(\epsilon^{-3})$ stochastic gradient and Hessian-vector product computations. Our algorithm uses Hessian-vector products to ``correct'' a bias term in the momentum of SGD with momentum. This leads to better gradient estimates in a manner analogous to variance reduction methods. In contrast to prior work, we do not require excessively large batch sizes,  and are able to provide an adaptive algorithm whose convergence rate automatically improves with decreasing variance in the gradient estimates. We validate our results on a variety of large-scale deep learning architectures and benchmarks tasks.
\end{abstract}
\section{Introduction}
First-order algorithms such as Stochastic Gradient Descent (SGD) or Adam \citep{kingma2014adam} have emerged as the main workhorse for modern Machine Learning (ML) tasks. They achieve good empirical results while being easy to implement and requiring relatively few computational resources. When the objective function is convex, SGD's convergence is well-understood \citep{zinkevich2003online}. However, the results are much less favorable for the non-convex setting in which modern deep learning models operate. In fact, the problem of optimizing non-convex function is NP-hard in general, so instead analysis often focuses on finding a critical point - that is, a point at which the gradient of the loss is zero. SGD is well-known to find an $\epsilon$-approximate critical point in at most $O(\epsilon^{-4})$ total stochastic gradient evaluations \citep{ghadimi2013stochastic}. In an effort to improve upon SGD, many different algorithms and heuristics have been proposed, including various advanced learning rate schedules \citep{loshchilov2016sgdr, goyal2017accurate}, or per-coordinate learning rates and adaptive algorithms \citep{duchi2011adaptive,mcmahan2010adaptive, kingma2014adam, reddi2018convergence}. All of these methods have enjoyed practical success, but none of their convergence rates has shown any asymptotic benefit over the $O(\epsilon^{-4})$ rate of SGD, which is to be expected because this rate is in fact \emph{optimal} in the worst-case for first-order methods operating on smooth losses \citep{arjevani2019lower}. Thus, in order to make improved algorithms, we need additional assumptions. In this work, we consider the case in which our algorithm is not a pure first-order method, but has some limited access to second-order information.

Second-order algorithms such as Newton's method are among the most powerful methods in optimization theory. Instead of using a linear approximation for the objective via the gradient, Newton's method employs a quadratic approximation using both the gradient and the Hessian. This quadratic approximation hugs the curvature of the error surface and allows each iteration to make much faster progress. Unfortunately, second-order algorithms are also much more complex and expensive to implement. The cost of forming the Hessian matrix is $O(d^2)$ and the cost of a Newton step is typically $O(d^3)$, thus rendering such algorithms largely impractical for large-scale learning problems \citep{bottou2007tradeoffs}. 

The good news is we may not have to explicitly compute the Hessian matrix to be able to take advantages of the nice properties of second-order methods. In particular, it is possible to compute a \emph{Hessian-vector product}, that is a vector of the form $\nabla^2 F(x)v$ for arbitrary $x$ and $v$, in roughly the same time it takes to evaluate $\nabla F(x)$
\citep{Pearlmutter94fastexact}. In this paper, we develop a novel second-order SGD-based algorithm that uses a stochastic gradient and Hessian-vector product oracle to expedite the training process. 
Our algorithm not only enjoys optimal theoretical properties, it is also practically effective, as demonstrated through our experimental results across various deep learning tasks.
\subsection{Contributions}
We present a novel algorithm based on SGD with momentum (SGDHess) that uses Hessian-vector products to ``correct'' a bias term in the momentum. Through our theoretical analysis, we show that our algorithm requires the optimal $O(\epsilon^{-3})$ oracle calls for finding stationary points \citep{pmlr-v125-arjevani20a}. In contrast to previous algorithms with this property (e.g. \citep{pmlr-v125-arjevani20a}), we do not require excessively large batch sizes, and we also feel that our algorithm and analysis is relatively more straightforward: our method is a simple modification to SGD that could easily be applied to any momentum-based optimizer.

We also provide a variant of our algorithm based on normalized SGD, which dispenses with a Lipschitz assumption on the objective, and another variant with  an \emph{adaptive} learning rate that automatically improves to a rate of $O(\epsilon^{-2})$ when the noise in the gradients is negligible.

Finally, we test our algorithm on multiple learning tasks on different deep architectures. In all of these tasks, our algorithm consistently matches or exceeds the best algorithms for the task. Further, the tuning process of our algorithm is reasonably simple; in many cases, we can use the exact parameters of SGD for the new algorithm.

In the next two subsections (\ref{sec:related} and \ref{sec:setup}), we describe some related work and the assumptions and formal problem setting we study. Then, in Section \ref{sec:algorithm1}, we formally provide our algorithm and discuss its theoretical lower bound. In Section \ref{sec:normalized} and Section \ref{adaptive}, we provide an analysis of our algorithm in the normalized and adaptive settings. The empirical results are discussed in Section 5 and we conclude with a discussion on Section 6.

\subsection{Related works}\label{sec:related}
Although second-order methods have many attractive theoretical properties, making these methods practical is challenging. There have been multiple efforts to develop efficient second-order algorithms. One popular approach is the Broyden–Fletcher–Goldfarb–Shanno algorithm (BFGS) \citep{keskar2019limited, Liu89onthe,bollapragada2018progressive, pan2017accelerating}, which provides a faster approximation to the Newton step using only a first-order oracle. Although BFGS and similar methods have been applied with some success to deep learning tasks \citep{bollapragada2018progressive, ma2020apollo}, we stress that such methods are \emph{fundamentally incapable} of matching the convergence guarantees available to algorithms that truly use the Hessian in the stochastic non-convex setting \citep{arjevani2019lower}.


Another line of work is the use of trust region methods \citep{xu2020second} incorporating Hessian information using Hessian Sub-Sampling \citep{xu2020newton}. Although the method is theoretically appealing, it is computationally expensive and thus is less practical on very large networks. An alternative direction is to incorporate a diagonal approximation of the Hessian \citep{yao2020adaHessian} using Hutchinson's method. The advantage of this method is its $O(d)$ memory complexity and its ability to perform well on popular deep learning benchmarks. However, it does not have as strong a theoretical basis as other methods.


More recent attention has focused on the use of Hessian-vector products as a way to achieve some of the benefits of second-order information without incurring significant computational overhead \citep{carmon2018accelerated,agarwal2017finding}.  \citet{tripuraneni2017stochastic} uses Hessian-vector products to approximate the cubic regularized Newton method \citep{nesterov2006cubic}, converging in only $O(\epsilon^{-3.5})$ stochastic oracle calls. \citet{JMLR:v20:19-055} leverages both Hessian information and variance-reduction techniques to achieve a convergence rate of $O(\frac{n^{4/5}}{\epsilon^{3/2}})$ for finite-sum problems with $n$ summands. However, both algorithms fail to achieve the optimal $O(\epsilon^{-3})$ rate, which was first achieved by \citet{arjevani2020second}, and to our knowledge none of these have been tested extensively on deep learning benchmarks.
\subsection{Problem setup}\label{sec:setup}
We are interested in minimizing a function $F$ given by:
\begin{align*}
    F(\bx) = \E_{z \sim P_z}[f(\bx, z)]
\end{align*}
Where $f(\bx, z)$ is a differentiable function of $\bx$. We are also given a point $\bx_1$ and define $\Delta\in  \R$ by:
\begin{align}
    \sup_{\bx \in \R^d} F(\bx_1) - F(\bx) = \Delta \label{eqn:diameter}
\end{align}
Further, we assume $F$ is $L$-smooth. That is, for all $\bx$ and $\by$:
\begin{align}
    \|\nabla F(\bx) - \nabla F(\by) \| \le L\|\bx - \by \| \label{eqn:Lsmooth}
\end{align}

This setup models the standard supervised learning setting. In this scenario, we are interested
in finding some model parameters that minimize a population loss function. Thus, $\bx$ indicates the model parameters (e.g. the
weights of some neural networks), $z$ indicates an example data point\footnote{$z$ could  also indicate a minibatch of  examples}, (e.g. audio file/text transcript pair), and
$f(\bx, z)$ indicates the loss of the neural network using weights $\bx$ on the example $z$. The distribution $P_z$ is some
distribution over examples. It could be an empirical distribution over training examples, or the true distribution over examples in the wild.

The major challenge here is that we may not know the distribution $P_z$. As a result, it is usually either impossible or computationally unreasonable for us to
actually compute the true value of $F$. On top of that, without further
assumptions like convexity, it is NP-hard to find the minimizer of $F$. Instead, we will search for a \emph{critical point}, which is a point where the gradient is 0. Thus, our goal is to develop an
algorithm that will output an $\bx$ that makes $E[\|\nabla F(\bx)\|]$ as small as possible, where the expectation is over both
the randomness in the gradient queries as well any randomness in the algorithm.

To design a second-order algorithm, we also assume access to second-order stochastic oracle. Given any $\bx$ and vector $\bw$, we are allowed to sample $z \sim P_z$ and compute $\nabla f(\bx, z)$ and $\nabla^2f(\bx,z)\bw$. We assume that for all $\bx$ and $\bw$:
\begin{align}
    \E[\|\nabla f(\bx, z) -\nabla F(\bx)\|^2] &\le \sigma_G^2 \label{eqn:gradvariance}\\
    \E[\|\nabla^2f(\bx,z)\bw -\nabla^2F(\bx,z)\bw\|^2] &\le \sigma_H^2\|\bw\|^2 \label{eqn:hessvectorvariance}
\end{align}

In addition, we will also need the assumption of second-order smoothness. For all $\bx$, $\by$, and $\bw$:
\begin{align}
    \|(\nabla^2F(\bx)-\nabla^2F(\by))\bw\| \le \rho \|\bx-\by\| \|\bw\| \label{eqn:rhosmoothness}
\end{align}

Finally, for technical reasons in the analysis of two of  our algorithms, we will need to assume a crude bound on the magnitude of gradients of $f$. Suppose there exists some $G$ such that:
\begin{align}
    \|\nabla f(\bx,z)\| \le G \label{eqn:gradientbound}
\end{align}
\section{SGD with Hessian-corrected momentum}\label{sec:algorithm1}
In this section, we provide our SGD-based second-order algorithm. Before we go into the analysis, let us compare our algorithm to SGD to see why adding the Hessian-vector product term to the momentum update is a good idea.
The standard SGD with momentum update is the following:
\begin{align*}
    \hat g_t &= (1-\alpha)\hat g_{t-1} + \alpha\nabla f(\bx_t,z_t)\\
    \bx_{t+1} &= \bx_t - \eta\hat g_t
\intertext{To gain some intuition for this update, let us suppose (for illustrative purposes only) that $\hat g_{t-1} = \nabla F(\bx_{t-1})$. Then, viewing $\hat g_t$ as an estimate of $\nabla F(\bx_t)$, define the error as:}
\hat\epsilon_t &= \hat g_t - \nabla F(\bx_t)
\end{align*}
Intuitively, SGD will converge rapidly if this error is small. We can write:
\begin{align*}
    \hat\epsilon_t = (1-\alpha)(\hat g_{t-1} - \nabla F(\bx_t))& 
  + \alpha(\nabla f(\bx_t, z_t)- \nabla F(\bx_t))
\end{align*}
The second term is fairly benign: it is zero in expectation and is multiplied by a potentially small value $\alpha$. However, the first term is a bit trickier to bound since $\nabla F(\bx_{t-1}) \ne \nabla F(\bx_{t})$. Our approach is to leverage the second-order oracle to improve this estimate. Specifically, we modify the momentum update to:
\begin{align*}
    \hat g_t &= (1-\alpha)(\hat g_{t-1} + \nabla^2 f(\bx_t, z_t)(\bx_t-\bx_{t-1})) + \alpha\nabla f(\bx_t,z_t)
\end{align*}
Now, if $\hat g_{t-1} = \nabla F(\bx_{t-1})$, we would have
\begin{align*}
    \E[\hat g_{t-1} + \nabla^2 f(\bx_t, z_t)(\bx_t-\bx_{t-1})] &= \hat g_{t-1} + \nabla^2 F(\bx_t, z_t)(\bx_t-\bx_{t-1})\\
&= \nabla F(\bx_t)+ O(\|\bx_t-\bx_{t-1}\|^2)
\end{align*}
Further, we have:
\begin{align*}
    \hat\epsilon_t = (1-\alpha)(\hat g_{t-1} - \nabla F(\bx_t) + \nabla^2f(\bx_t, z_t)(\bx_t - \bx_{t-1}))
  + \alpha(\nabla f(\bx_t, z_t)- \nabla F(\bx_t))
\end{align*}
so that the first term is now bounded by $O(\|\bx_{t-1}-\bx_t\|^2)$ and the second term is again easy to control via tuning $\alpha$. Without this second-order correction, we need to rely on $\hat g_{t-1} = \nabla F(\bx_t)+ O(\|\bx_t-\bx_{t-1}\|)$. Thus, by using another term of the Taylor expansion, we have improved the dependency to $O(\|\bx_t-\bx_{t-1}\|^2)$, which will create a corresponding reduction in our final error. The idea is that using momentum should make $\hat g_t$ very close to $\nabla F(x_t)$ because momentum is averaging together many estimates. Thus, to gain some intuition about how the estimate changes from one iteration to the next, we examined the case that the prior iteration's estimate was actually correct. Of course, in our theoretical analysis we \emph{do not} make any such assumption.

Note that this approach is morally similar to the methods of \citep{cutkosky2019momentum,tran2019hybrid}, which are themselves similar to the recursive variance reduction \citep{nguyen2017sarah} algorithm for stochastic convex optimization. In these algorithms, the Hessian-vector product is replaced with two gradient evaluations with the same example $z_t$: $\nabla f(\bx_t,z_t)- \nabla f(\bx_{t-1},z_t) $. So long as the individual functions $f(\bx_t,z_t)$ are $L$-smooth, this will eventually have a similar correction effect. However, past empirical work suggests that variance reduction may not be effective in practice on deep learning tasks \citep{defazio2019inefectiveness}, while the use of Hessian-vector products has not been as extensively tested to our knowledge.

Our stochastic gradient descent algorithm with Hessian-corrected momentum is the following, and its analysis is presented in Theorem \ref{thm:sgdhess}. Note that we only need to  make a small modification to the standard SGD update, which allows for streamlined analyses. Concretely, we avoid the large batch-size requirement of \citep{arjevani2020second}, and can extend the analysis to adaptive learning rates in Section \ref{adaptive}.
\begin{algorithm}
   \caption{SGD with Hessian-corrected Momentum (\textbf{SGDHess})}
   \label{alg:sgdhess}
   \begin{algorithmic}
      \STATE{\bfseries Input: } Initial Point $\bx_1$, learning rates $\eta_t$, momentum parameters $\alpha_t$, time horizon $T$, parameter $G$:
      \STATE Sample $z_1\sim P_z$.
      \STATE $\hat g_1 \leftarrow \nabla f(\bx_1, z_1)$.
      \STATE $\bx_2 \leftarrow \bx_1 - \eta_1\hat g_1$
      \FOR{$t=2\dots T$}
      \STATE Sample $z_t\sim P_z$.
      \STATE  $\hat g_t \leftarrow (1-\alpha_{t-1}) (\hat g^{clip}_{t-1} + \nabla^2 f(\bx_t, z_t)(\bx_t - \bx_{t-1}))+ \alpha_{t-1}\nabla f(\bx_t, z_t)$.
      \STATE  $\hat g^{clip}_t \leftarrow  \hat g_t$ if $\|\hat g_t\| \le G$; otherwise, $\hat g^{clip}_t \leftarrow G\frac{\hat g_t}{\|\hat g_t\|}$
      
      \STATE $\bx_{t+1} \leftarrow \bx_t  - \eta_t \hat g^{clip}_{t}$.
      \ENDFOR
      \STATE Return $\hat x$ uniformly at random from $\bx_1,\dots,\bx_T$ (in practice $\hat x = \bx_{T}$).
   \end{algorithmic}
\end{algorithm}

\begin{Theorem}\label{thm:sgdhess}Assume (\ref{eqn:diameter}), (\ref{eqn:Lsmooth}), (\ref{eqn:hessvectorvariance}), (\ref{eqn:gradientbound}), (\ref{eqn:rhosmoothness}). Then Algorithm \ref{alg:sgdhess} with $\eta_t = \frac{1}{C t^{1/3}}$, $\alpha_t=2K\eta_t\eta_{t+1}$ with $C\ge \sqrt{2K}$ and $C \ge 4L$, $K = \frac{2G^2\rho^2}{-2\sigma_H^2+\sqrt{4\sigma_H^4+\frac{\rho^2G^2}{2}}}$ and $ D = \frac{24}{5}K + \frac{16C^6}{25 K^2}$ guarantees 
\begin{align*}
    &\frac{1}{T}\sum_{t=1}^T \E[\|\nabla F(\bx_t)\|^2]\le \frac{20C\Delta +96C^2G^2/K}{T^{2/3}} + \frac{20G^2D(1+\log(T))}{C^2T^{2/3}}
\end{align*}
\end{Theorem}
To see why we want to use a clipped gradient, let us consider the error term with an \emph{unclipped} gradient:
\begin{align*}
     \hat\epsilon_{t+1} &= (1-\alpha_t)(\hat g_{t} - \nabla F(\bx_t)) + (1-\alpha_t)(\nabla^2 f(\bx_{t+1}, z_{t+1})-\nabla^2 F(\bx_{t+1})(\bx_{t+1}-\bx_t))\\
&\quad\quad+
(1-\alpha_t)(\nabla F(\bx_t) +  \nabla^2F(\bx_{t+1})(\bx_{t+1}-\bx_t) -\nabla F(\bx_{t+1})) + \alpha_t \epsilon^G_{t+1}
\end{align*}
It is easy to see that we could bound the second and the fourth term using assumption (\ref{eqn:hessvectorvariance}) and (\ref{eqn:gradientbound}) and the first term is simply $(1-\alpha_t)\hat \epsilon_t$, which we can use to analyze how the error changes over each iteration . However, the third term is a bit trickier to control. Let $\delta_t =(\nabla F(\bx_t) +  \nabla^2F(\bx_{t+1})(\bx_{t+1}-\bx_t) -\nabla F(\bx_{t+1}))$, from second-order smoothness we have:
\begin{align*}
    \|\delta_t\|^2 \le \frac{\rho^2}{4}\|\bx_{t+1} - \bx_t\|^4 
    \le \frac{\rho^2}{4} \eta_t^4 \|\hat g_{t}\|^4
\end{align*}
Intuitively, if we let $\hat g_t$ be unbounded, $\|\hat g_t\|^4$ might be difficult to control since we only assume a bound on the variance of the $\nabla f(\bx_t, z_t)$ rather than the fourth moment. Therefore, by enforcing some bound on the norm of $\hat g_t$ (using clipping), we make sure that we can control this term properly.

In order to prove this Theorem, we will require two Lemmas. The first (Lemma \ref{thm:onestep}) is due to \citep{cutkosky2019momentum}, and provides a bound on the progress of one iteration of stochastic gradient descent without making any assumptions (such as unbiasedness) about the gradient estimates. The second (Lemma \ref{thm:sgderror}), is a technical result characterizing the quality of the gradient estimates $\hat g^{clip}_t$ generated by Algorithm \ref{alg:sgdhess}. 
\begin{restatable}{Lemma}{onestep}[\citep{cutkosky2019momentum} Lemma 2] \label{thm:onestep}
Define:
\begin{align*}
    \hat \epsilon_t &= \hat g^{clip}_{t} - \nabla F(\bx_t)
\end{align*}
Suppose $\eta_t$ is a deterministic and non-increasing choice of learning rate. Then, so long as $\eta_t\le \frac{1}{4L}$,
\begin{align*}
    \E[F(\bx_{t+1})-F(\bx_t)]&\le -\frac{\eta_t}{4} \E[\|\nabla F(\bx_t)\|^2] + \frac{3\eta_t}{4}\E[\|\hat\epsilon_t\|^2]
\end{align*}
\end{restatable}
\begin{restatable}{Lemma}{sgderror}
\label{thm:sgderror}
Suppose that $f(\bx,z)$ satisfies (\ref{eqn:Lsmooth}), (\ref{eqn:hessvectorvariance}), (\ref{eqn:gradientbound}), and (\ref{eqn:rhosmoothness}). Define:
\begin{align*}
    \hat \epsilon_t &= \hat g^{clip}_{t} - \nabla F(\bx_t)
\end{align*}
Now, for some constant $C$ and $\sigma_H$, set $K = \frac{2G^2\rho^2}{-2\sigma_H^2+\sqrt{4\sigma_H^4+\frac{\rho^2G^2}{2}}}$, $\eta_t = \frac{1}{Ct^{1/3}}$ with $C \geq \sqrt{2K}$, and $\alpha_t=2K\eta_t\eta_{t+1}$. Then we have:
\begin{align*}
     \frac{6}{5K\eta_{t+1}} \E[\|\hat\epsilon_{t+1}\|^2]-\frac{6}{5K\eta_t}\E[\|\hat\epsilon_t\|^2]&\le-\frac{3\eta_t}{4}\E[\|\hat\epsilon_t\|^2] + \frac{\eta_t}{5}\E[\|\nabla F(\bx_t)\|^2] \\
&\quad\quad\quad+ \eta_t^3\left(\frac{24}{5}KG^2
     +\frac{16C^6G^2}{25K^2}\right)
\end{align*}
\end{restatable}

Let us look into how Lemma \ref{thm:sgderror} is used in our analysis. To prove Theorem \ref{thm:sgdhess}, we use the Lyapunov function defined as $\Phi_t = F(\bx_t) +\frac{6}{5K\eta_t}\|\hat\epsilon_t\|^2$ to bound the $\E[\|\nabla F(\bx_t)\|^2]$ term. Then:
\begin{align}
     \E[\Phi_{t+1}-\Phi_t]&= \E\left[F(\bx_{t+1})-F(\bx_t) + \frac{6}{5K\eta_{t+1}}\|\hat\epsilon_{t+1}\|^2-\frac{6}{5K\eta_t}\|\hat\epsilon_t\|^2\right] \nonumber  \\
     &\le  \E\left[-\frac{\eta_t}{4}\|\nabla F(\bx_t)\|^2 + \frac{3\eta_t}{4}\|\hat\epsilon_t\|^2+ \frac{6}{5K\eta_{t+1}}\|\hat\epsilon_{t+1}\|^2-\frac{6}{5K\eta_t}\|\hat\epsilon_t\|^2\right] \label{potential}
\end{align}
where the inequality comes from Lemma \ref{thm:onestep}. Now if we plug in the result of Lemma \ref{thm:sgderror}, we are able to simplify the bound of (\ref{potential}) by canceling the positive error term $\frac{3\eta_t}{4}\|\hat \epsilon_t\|^2$ while keeping the coefficient of $\|\nabla F(\bx_t)\|^2$ negative. Then, we can move the negative term to the left hand side and derive a bound for $\E[\|\nabla F(\bx_t)\|^2]$. With Lemma \ref{thm:onestep} and Lemma \ref{thm:sgderror} in hand, we are now ready to prove Theorem \ref{thm:sgdhess}.


\begin{proof}[Proof of Theorem \ref{thm:sgdhess}]
We define the potential:
\begin{align*}
    \Phi_t = F(\bx_t) +\frac{6}{5K\eta_t }\|\hat\epsilon_t\|^2
\end{align*}
We will then show that $\Phi_t$ roughly decreases with $t$ at rate that depends on $\|\nabla F(\bx_t)\|^2$. Specifically:
\begin{align*}
    \E[\Phi_{t+1}-\Phi_t]&= \E\left[F(\bx_{t+1})-F(\bx_t) + \frac{6}{5K\eta_{t+1} }\|\hat\epsilon_{t+1}\|^2-\frac{6}{5K\eta_t }\|\hat\epsilon_t\|^2\right]
    \intertext{applying Lemmas \ref{thm:onestep} and \ref{thm:sgderror}:}
    &\le -\frac{\eta_t}{4}\E[\|\nabla F(\bx_t)\|^2] + \frac{3\eta_t}{4}\E[\|\hat\epsilon_t\|^2] -\frac{3\eta_t}{4}\E[\|\hat\epsilon_t\|^2]\\
    &\quad\quad + \frac{\eta_t}{5}\E[\|\nabla F(\bx_t)\|^2] + \eta_t^3G^2\left(\frac{24}{5}K + \frac{16C^6}{25 K^2}\right)\\
    &=\frac{-\eta_t}{20}\E[\|\nabla F(\bx_t)\|^2] + \eta_t^3 G^2(\frac{24}{5}K + \frac{16C^6}{25 K^2})
    \intertext{Let $ D = \frac{24}{5}K + \frac{16C^6}{25 K^2}$:}
    &= \frac{-\eta_t}{20}\E[\|\nabla F(\bx_t)\|^2] + \eta_t^3G^2D
    \intertext{Now sum over $t$, and use $\eta_t\ge \eta_T$ for all $t$:}
    \E[\Phi_{T+1} - \Phi_1]&\le \frac{-\eta_T}{20}\sum_{t=1}^T \E[\|\nabla F(\bx_t)\|^2] +  G^2D\sum_{t=1}^T \eta_t^3\\
     \sum_{t=1}^T \E[\|\nabla F(\bx_t)\|^2] &\le \frac{20}{\eta_T}\E[\Phi_1-\Phi_{T+1}] + \frac{20G^2D\sum_{t=1}^T \eta_t^3}{\eta_T}
\intertext{Now, observe that:}
    \E[\Phi_1-\Phi_{T+1}]&= \E[F(\bx_1)-F(\bx_{T+1}) + \frac{6}{5K\eta_1}\|\hat\epsilon_1\|^2 - \frac{6}{5K\eta_{T+1}}\|\hat\epsilon_{T+1}\|^2]\\
    &\le \Delta + \frac{24G^2}{5K\eta_1} \le \Delta +\frac{24CG^2}{5K}
\intertext{where in the second line we used $\E[\|\hat\epsilon_1\|^2]\le 4G^2$. Also, we have:}
\sum_{t=1}^T \eta_t^3 &= \frac{1}{C^3}\sum_{t=1}^T\frac{1}{t}\le \frac{1+\log(T)}{C^3}
\intertext{Putting all this together yields:}
    \frac{1}{T}\sum_{t=1}^T \E[\|\nabla F(\bx_t)\|^2] &\le \frac{20}{T\eta_T}\E[\Phi_1-\Phi_{T+1}] + \frac{20G^2D\sum_{t=1}^T \eta_t^3}{T\eta_T}\\
    &\le \frac{20C\Delta +96C^2G^2/K}{T^{2/3}} + \frac{20G^2D(1+\log(T))}{C^2T^{2/3}}
\end{align*}
\end{proof}
\section{Normalized SGD with Hessian-corrected momentum }\label{sec:normalized}
\begin{algorithm}
   \caption{Normalized SGD with Hessian-corrected Momentum (\textbf{N-SGDHess})}
   \label{alg:normalized}
   \begin{algorithmic}
      \STATE{\bfseries Input: } Initial Point $\bx_1$, learning rates $\eta$, momentum parameters $\alpha$, time horizon $T$, parameter $G$:
      \STATE Sample $z_1\sim P_z$.
      \STATE $\hat g_1 \leftarrow \nabla f(\bx_1, z_1)$.
      \STATE  $\bx_2 \leftarrow \bx_1 - \eta\frac{\hat g_1}{\|\hat g_1\|}$
      \FOR{$t=2\dots T$}
      \STATE Sample $z_t\sim P_z$.
      \STATE  $\hat g_t \leftarrow (1-\alpha) (\hat g_{t-1} + \nabla^2 f(\bx_t, z_t)(\bx_t - \bx_{t-1}))+ \alpha\nabla f(\bx_t, z_t)$.
     \STATE  $\bx_{t+1} \leftarrow \bx_t  - \eta \frac{\hat g_t}{\|\hat g_t\|}$.
      \ENDFOR
      \STATE Return $\hat x$ uniformly at random from $\bx_1,\dots,\bx_T$ (in practice $\hat x = \bx_T$).
   \end{algorithmic}
\end{algorithm}

In this section, we introduce an algorithm that dispenses with the assumption (\ref{eqn:gradientbound}) required by Algorithm \ref{alg:sgdhess}. This method (Algorithm \ref{alg:normalized}) uses SGD with normalized updates and Hessian-vector product-based momentum. We will show that normalization can significantly simplify the analysis from section \ref{sec:algorithm1} while still maintaining $O(\epsilon^{-3})$ convergence rate. Indeed, the technical term in the analysis of Algorithm \ref{alg:sgdhess} that required us to enforce a bound on the updates via clipping simply does not appear because the updates are automatically bounded by $\eta$.

To get started, we need the following Lemma, which is essentially identical to Lemma 2 of \citep{cutkosky2020momentum}, with slightly improved constants. We provide the proof in the Appendix for completeness.
\begin{restatable}{Lemma}{normgradient}
\label{normgradient}
Define:
\begin{align*}
    \hat \epsilon_t = \hat g_t - \nabla F(\bx_t)
\end{align*}
Suppose $\bx_1, \dots, \bx_T$ is a sequence of iterates defined by $\bx_{t+1} = \bx_t - \frac{\hat g_t}{\|\hat g_t\|} $ for some arbitrary sequence $\hat g_1,\dots,\hat g_T$. Then if $\bx_t$ is chosen uniformly at random from $\bx_1, \dots, \bx_T$, we have:
\begin{align*}
    \E[\|\nabla F(\bx_t)\|] &\le \frac{3\Delta}{2\eta T} + \frac{3L\eta}{4} +\frac{3}{T}\sum_{t=1}^{T} \|\hat\epsilon_t\|
\end{align*}
\end{restatable}
Now, we are able to apply this Lemma to analyze Algorithm \ref{alg:normalized}:
\begin{restatable}{Theorem}{normalizedhess}\label{normalizedhess}
Assuming (\ref{eqn:diameter}), (\ref{eqn:Lsmooth}), (\ref{eqn:gradvariance}), (\ref{eqn:hessvectorvariance}), and (\ref{eqn:rhosmoothness}) hold (but \emph{not} assuming (\ref{eqn:gradientbound})), with $ \alpha = \min \{ \max \{ \frac{1}{T^{2/3}}, \frac{\Delta^{4/5}\rho^{2/5}}{T^{4/5}\sigma_G^{6/5}}, \frac{(2\Delta\sigma_H)^{2/3}}{T^{2/3}\sigma_G^{4/3}} \}, 1 \}$ and $\eta = \min \{ \frac{\sqrt{2\Delta}\alpha^{1/4}}{\sqrt{T(L\sqrt{\alpha}+4\sigma_H})}, \frac{(\Delta\alpha)^{1/3}}{(\rho T)^{1/3}} \}$,   Algorithm \ref{alg:normalized} guarantees
\begin{align*}
    \E[\|\nabla F(\bx_t)\|] &\le \frac{6\sigma_G+ 54^{1/3}(\Delta \sigma_H)^{1/3}\sigma_G^{1/3}}{T^{1/3}} +\frac{6\sigma_G^{2/5}\Delta^{2/5}\rho^{1/5}}{T^{2/5}}\\
    &\quad\quad +\frac{\sqrt{9\Delta L} + \sqrt{72\Delta\sigma_H}}{\sqrt{T}} +\frac{6\Delta^{2/3}\rho^{1/3}}{T^{2/3}}+ \frac{\sqrt{18\Delta\sigma_H}}{T^{3/2}} + \frac{3\Delta^{2/3}\rho^{1/3}}{T^{5/3}}
\end{align*}
In words, Algorithm \ref{alg:normalized} achieves $O(1/T^{1/3})$ with large $\sigma_H$ and $\sigma_G$, and achieves $O(1/\sqrt{T})$ in noiseless case, without requiring a Lipschitz bound on the objective.

\end{restatable}
\section{Adaptive SGD with Hessian-corrected Momentum} \label{adaptive}

\begin{algorithm}
   \caption{Adaptive learning rate for SGD with Hessian-corrected Momentum}
   \label{alg:adaptive}
  \begin{algorithmic}
      \STATE{\bfseries Input: } Initial Point $\bx_1$, parameters $c$, $w$,  $\alpha_t$, time horizon $T$, parameter $G$:
      \STATE Sample $z_1\sim P_z$.
      \STATE $\hat g_1 \leftarrow \nabla f(\bx_1, z_1)$
      \STATE  $G_1 \leftarrow \|\nabla f(\bx_1, z_1)\|$.
      \STATE $\eta_1 \leftarrow \frac{c}{w^{1/3}}$
      \STATE $\bx_2 \leftarrow \bx_1 - \eta_1\hat g_1$
      \FOR{$t=2\dots T$}
      \STATE Sample $z_t\sim P_z$.
      \STATE  $G_1 \leftarrow \|\nabla f(\bx_t, z_t)\|$
      \STATE $\hat g_t \leftarrow (1-\alpha_{t-1}) (\hat g^{clip}_{t-1} + \nabla^2 f(\bx_t, z_t)(\bx_t - \bx_{t-1}))+ \alpha_{t-1}\nabla f(\bx_t, z_t)$.
      \STATE  $\hat g^{clip}_t \leftarrow  \hat g_t$ if $\|\hat g_t\| \le G$; otherwise, $\hat g^{clip}_t \leftarrow G\frac{\hat g_t}{\|\hat g_t\|}$
    \STATE  $\eta_t \leftarrow \frac{c}{(w+\sum_{i=1}^{t-2}G_i^2)^{1/3}}$ \text{\ \ \ \ (set $\eta_2=\frac{c}{w^{1/3}}$)}.
      \STATE $\bx_{t+1} \leftarrow \bx_t  - \eta_t \hat g^{clip}_{t}$.
      \ENDFOR
      \STATE Return $\hat x$ uniformly at random from $\bx_1,\dots,\bx_T$ (in practice $\hat x = \bx_T$).
   \end{algorithmic}
\end{algorithm}
In previous sections, we have presented two different versions of our algorithm. Both algorithms achieve the worst-case optimal $O(1/T^{1/3})$ convergence rate. However, the bounds are non-adaptive and remain $O(1/T^{1/3})$ even in noiseless settings (e.g. if $\sigma_G=0$). The bound for the normalized SGD algorithm of Section \ref{sec:normalized} is better in the sense that when the noise is negligible, we provide an explicit tuning of the learning rate to achieve $O(1/\sqrt{T})$, but this requires us to set the parameters $\eta$ and $\alpha$ based on prior knowledge of $\sigma_G$. In this section, we will describe an \emph{adaptive} version of our algorithm that has the best of both worlds. The algorithm doesn't require the knowledge of $\sigma_G$ but still automatically improves to a tighter bound whenever $\sigma_G$ is small. 
\begin{restatable}{Theorem}{adaptive}\label{thm:adaptive}
With $K = \frac{2G^2\rho^2}{-2\sigma_H^2+\sqrt{4\sigma_H^4+\frac{\rho^2G^2}{2}}}$, $\eta_t = \frac{c}{(w+ \sum_{i=1}^{t-2} G_i^2)^{1/3}}$ with $c \le \frac{2G^{2/3}}{\sqrt{K}}$ and $w = \max\{(4Lc)^3, 3G^2\}$, $\alpha_t = 2K\eta_t\eta_{t+1}$. Algorithm \ref{alg:adaptive} guarantees:
\begin{align*}
    \E\left[\frac{1}{T}\sum_{t=1}^T\|\nabla F(\bx_t)\|\right] \le \frac{w^{1/6}\sqrt{2M}+2M^{3/4}}{\sqrt{T}} +\frac{2\sigma_G^{1/3}}{T^{1/3}}
\end{align*}
Where $M = \frac{1}{c}\left(20(\Delta + \frac{6\sigma_G^2w^{1/3}}{5Kc}) + 96Kc^2\ln(T+1)+  \frac{64G^4}{5K^2c^3} \ln T\right)$.
\end{restatable}
In words, algorithm \ref{alg:adaptive} converges with $O\left(\frac{1}{T^{1/3}}\right)$ rate in noisy case and automatically improves to $O\left(\frac{\ln T}{\sqrt{T}}\right)$ in noiseless case. We defer the proof of Theorem \ref{adaptive} to the appendix.
\section{Experiments}\label{sec:experimental}
\subsection{Setup}
Our algorithm is a simple extension of the official SGD implementation in Pytorch. The Hessian-vector products can be efficiently computed using the automatic differentiation package \citep{paszke2017automatic}: $\nabla^2f(\bx, z)v = \nabla h(\bx)$ where $h(\bx) = \langle \nabla f(\bx,z), v\rangle$. Since Pytorch allows backprogation through the differentiation process itself, this is straightforward to implement. To validate the effectiveness of our proposed algorithm, we perform experiments in two tasks: image classification and neural machine translation on popular deep learning benchmarks. The performance of SGDHess is compared to those of commonly used optimizers such as Adam and SGD as well as AdaHessian, another algorithm that incorporates second-order information. All experiments are run on NVIDIA v100 GPUs.
\begin{figure*}[ht!]
\centering
  \begin{subfigure}
     \centering
     \includegraphics[width=0.45\textwidth]{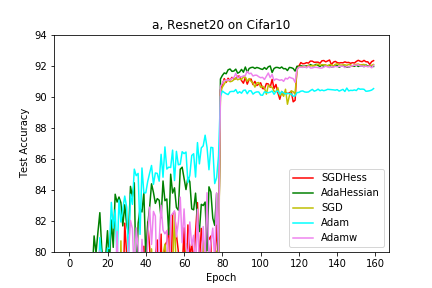}
     \label{fig:resnet20}
    \qquad
  \end{subfigure}
      \begin{subfigure}
     \centering
     \includegraphics[width=0.45\textwidth]{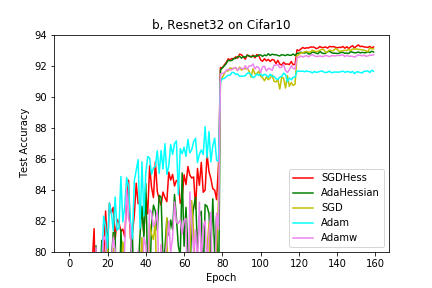}
     \label{fig:resnet32}
  \end{subfigure}
\caption{ (a) SGDHess (Red) achieves the best accuracy among all optimizers (92.46 $\pm 0.07 \%$). (b) Similar to the experiment on Resnet20, SGDHess also converges to the best accuracy (93.19 $\pm 0.08 \% $). For the full plot, refer to the appendix.
    }
  \label{fig:experiments}
\end{figure*}
\subsection{Image Classification}
Our Cifar10 experiment is conducted using the official implementation of AdaHessian. For AdaHessian, we use their recommended values for all the parameters. For the rest of the optimizers, we performed a grid search on the base learning rate $\eta \in \{ 0.001, 0.005, 0.01, 0.05, 0.1, 0.5, 1 \}$ to find the best settings. Similar to the Cifar10 experiment of AdaHessian, we also trained our models on 160 epochs and we ran each optimizer 5 times and reported the average best accuracy as well as the standard deviation (detailed results in the appendix). As we can see from the results in Figure \ref{fig:experiments}, SGDHess outperforms all other optimizers (0.32$\%$ and 0.11$\%$ better than the next best optimizer in Resnet20 and Resnet32 respectively).
 
\begin{table}[h]
  \caption{Top 1 accuracy on Imagenet}
  \break
  \label{imagenet}
  \centering
 \begin{tabular}{|c| c| c|} 
 \hline
SGD & SGDHess & AdaHessian  \\ 
 \hline\hline
70.36 & 70.58 & 69.89  \\ 
 \hline
 \end{tabular}
\end{table}

We also train SGD, SGDHess, and AdaHessian with Imagenet \citep{deng2009imagenet} on Resnet18 to see how well SGDHess perform on a larger-scale benchmark.  We use standard parameter values for SGD (lr = 0.1, momentum = 0.9, weight$\_$decay = 1e-4) for both SGD and SGDHess and the recommended parameters values for AdaHessian. For the learning rate scheduler, we employ the plateau decay scheduler that was used in \citep{yao2020adaHessian}. We train our model in 90 epochs as usual. Even without extensive tuning, SGDHess not only still outperforms SGD (as shown in Table \ref{imagenet}) but also comes close to the state-of-the-art accuracy (70.7$\%$) on this particular task \citep{darknet13}, even though these settings were chosen with SGD in mind rather than SGDHess.

\subsection{Neural Machine Translation}

\begin{table}[h!]
\caption{Bleu Score on IWSLT'14}
\label{bleu}
\centering
 \begin{tabular}{|c| c| c| c|} 
 \hline
SGD & SGDHess & AdamW & AdaHessian \\ 
 \hline\hline
29.75 & 33.73 & 33.95 & 33.62 \\ 
 \hline
 \end{tabular}
\end{table}
We use the IWSLT'14 German to English dataset that contains 153k/7k/7k in the train/validation/test set. Our experiments are run using all the default values specified in the official implementation of Fairseq \citep{ott2019fairseq}. We use BLEU \citep{papineni-etal-2002-bleu} as the evaluation metrics for our experiment. For AdaHessian, we use the parameters specified in \citep{yao2020adaHessian}. For other optimizers, we again run a grid search to find the best learning rate for the task. The best bleu scores are reported in Table \ref{bleu}. It is worth stressing that SGDHess is an algorithm based on SGD, which consistently performs much worse than adaptive algorithms such as AdaHessian and AdamW in this type of task. Still, SGDHess is able to produce comparable results to those of AdaHessian and AdamW, thus significant bridging the gap between SGD and other adaptive algorithms (an almost 4 points increase compared to SGD's in BLEU score, which is significant for the task).
\subsection{Discussion on run time}
Since SGDHess requires the computation of Hessian-vector product in every iteration to "correct" the momentum, it is inevitable that its run time is slower than that of first-order algorithms such as Adam or SGD. Fortunately, the penalty is small even in our unoptimized implementation. Specifically, for image classification task, SGDHess is roughly 1.7/1.6 times slower than SGD and Adam respectively (AdaHessian is 1.9/1.7 times slower than SGD/Adam). For NLP task, SGDHess is 1.3 times slower than SGD/Adam (AdaHessian is 1.7 times slower). Furthermore, the tuning of SGDHess is relatively straighforward. In a lot of cases, the optimal tuning of SGDHess is the same as that of SGD. Thus, to reduce the computation overhead, one could try tuning SGD first then using the optimal parameters of SGD for SGDHess.
\section{Conclusion and Future Work}\label{sec:conclusion}
In this paper, we have presented SGDHess, a novel SGD-based algorithm using Hessian-corrected momentum. We show that when the objective is second-order smooth, our algorithm can achieve the optimal $O(\epsilon^{-3})$ bound. Further, we provide a variation of our algorithm with normalized updates where the analysis is significantly simplified and we do not require a Lipschitz bound on the losses while still maintaining the optimal convergence rate. Finally, we provide experimental results on 3 different tasks in Computer Vision, Language Modeling, and Neural Machine Translation. In each of these tasks, SGDHess consistently performs better or comparable to other commonly used optimizers such as SGD and Adam. 


It is our hope that this work demonstrates that Hessian-based optimization can combine both theoretical and practical improvements for large-scale machine learning problems. Our modification to the standard momentum formula is extremely simple, and so we suspect that is possible to make similar modifications to other popular optimization algorithms that use momentum. For example, one might hope that an appropriate modification to Adam \citep{kingma2014adam} or AMSGrad \citep{reddi1904convergence} might yield an Hessian-based optimized that enjoys \emph{adaptive} convergence rates and improved empirical performance.

{\small
\bibliography{arxiv}

\begin{thebibliography}{37}
\providecommand{\natexlab}[1]{#1}
\providecommand{\url}[1]{\texttt{#1}}
\expandafter\ifx\csname urlstyle\endcsname\relax
  \providecommand{\doi}[1]{doi: #1}\else
  \providecommand{\doi}{doi: \begingroup \urlstyle{rm}\Url}\fi

\bibitem[Agarwal et~al.(2017)Agarwal, Allen-Zhu, Bullins, Hazan, and
  Ma]{agarwal2017finding}
Agarwal, N., Allen-Zhu, Z., Bullins, B., Hazan, E., and Ma, T.
\newblock Finding approximate local minima faster than gradient descent.
\newblock In \emph{Proceedings of the 49th Annual ACM SIGACT Symposium on
  Theory of Computing}, pp.\  1195--1199, 2017.

\bibitem[Arjevani et~al.(2019)Arjevani, Carmon, Duchi, Foster, Srebro, and
  Woodworth]{arjevani2019lower}
Arjevani, Y., Carmon, Y., Duchi, J.~C., Foster, D.~J., Srebro, N., and
  Woodworth, B.
\newblock Lower bounds for non-convex stochastic optimization.
\newblock \emph{arXiv preprint arXiv:1912.02365}, 2019.

\bibitem[Arjevani et~al.(2020{\natexlab{a}})Arjevani, Carmon, Duchi, Foster,
  Sekhari, and Sridharan]{arjevani2020second}
Arjevani, Y., Carmon, Y., Duchi, J.~C., Foster, D.~J., Sekhari, A., and
  Sridharan, K.
\newblock Second-order information in non-convex stochastic optimization: Power
  and limitations.
\newblock In \emph{Conference on Learning Theory}, pp.\  242--299,
  2020{\natexlab{a}}.

\bibitem[Arjevani et~al.(2020{\natexlab{b}})Arjevani, Carmon, Duchi, Foster,
  Sekhari, and Sridharan]{pmlr-v125-arjevani20a}
Arjevani, Y., Carmon, Y., Duchi, J.~C., Foster, D.~J., Sekhari, A., and
  Sridharan, K.
\newblock Second-order information in non-convex stochastic optimization: Power
  and limitations.
\newblock In Abernethy, J. and Agarwal, S. (eds.), \emph{Proceedings of Thirty
  Third Conference on Learning Theory}, volume 125 of \emph{Proceedings of
  Machine Learning Research}, pp.\  242--299. PMLR, 09--12 Jul
  2020{\natexlab{b}}.
\newblock URL \url{http://proceedings.mlr.press/v125/arjevani20a.html}.

\bibitem[Bollapragada et~al.(2018)Bollapragada, Nocedal, Mudigere, Shi, and
  Tang]{bollapragada2018progressive}
Bollapragada, R., Nocedal, J., Mudigere, D., Shi, H.-J., and Tang, P. T.~P.
\newblock A progressive batching l-bfgs method for machine learning.
\newblock In \emph{International Conference on Machine Learning}, pp.\
  620--629. PMLR, 2018.

\bibitem[Bottou \& Bousquet(2007)Bottou and Bousquet]{bottou2007tradeoffs}
Bottou, L. and Bousquet, O.
\newblock The tradeoffs of large scale learning.
\newblock In \emph{Proceedings of the 20th International Conference on Neural
  Information Processing Systems}, pp.\  161--168, 2007.

\bibitem[Carmon et~al.(2018)Carmon, Duchi, Hinder, and
  Sidford]{carmon2018accelerated}
Carmon, Y., Duchi, J.~C., Hinder, O., and Sidford, A.
\newblock Accelerated methods for nonconvex optimization.
\newblock \emph{SIAM Journal on Optimization}, 28\penalty0 (2):\penalty0
  1751--1772, 2018.

\bibitem[Cutkosky \& Mehta(2020)Cutkosky and Mehta]{cutkosky2020momentum}
Cutkosky, A. and Mehta, H.
\newblock Momentum improves normalized sgd.
\newblock \emph{arXiv preprint arXiv:2002.03305}, 2020.

\bibitem[Cutkosky \& Orabona(2019)Cutkosky and Orabona]{cutkosky2019momentum}
Cutkosky, A. and Orabona, F.
\newblock Momentum-based variance reduction in non-convex sgd.
\newblock \emph{Advances in neural information processing systems}, 32, 2019.

\bibitem[Defazio \& Bottou(2019)Defazio and Bottou]{defazio2019inefectiveness}
Defazio, A. and Bottou, L.
\newblock On the ineffectiveness of variance reduced optimization for deep
  learning.
\newblock In Wallach, H., Larochelle, H., Beygelzimer, A., d\textquotesingle
  Alch\'{e}-Buc, F., Fox, E., and Garnett, R. (eds.), \emph{Advances in Neural
  Information Processing Systems}, volume~32, pp.\  1755--1765. Curran
  Associates, Inc., 2019.
\newblock URL
  \url{https://proceedings.neurips.cc/paper/2019/file/84d2004bf28a2095230e8e14993d398d-Paper.pdf}.

\bibitem[Deng et~al.(2009)Deng, Dong, Socher, Li, Li, and
  Fei-Fei]{deng2009imagenet}
Deng, J., Dong, W., Socher, R., Li, L.-J., Li, K., and Fei-Fei, L.
\newblock Imagenet: A large-scale hierarchical image database.
\newblock In \emph{2009 IEEE conference on computer vision and pattern
  recognition}, pp.\  248--255. Ieee, 2009.

\bibitem[Duchi et~al.(2011)Duchi, Hazan, and Singer]{duchi2011adaptive}
Duchi, J., Hazan, E., and Singer, Y.
\newblock Adaptive subgradient methods for online learning and stochastic
  optimization.
\newblock \emph{Journal of machine learning research}, 12\penalty0 (7), 2011.

\bibitem[Ghadimi \& Lan(2013)Ghadimi and Lan]{ghadimi2013stochastic}
Ghadimi, S. and Lan, G.
\newblock Stochastic first-and zeroth-order methods for nonconvex stochastic
  programming.
\newblock \emph{SIAM Journal on Optimization}, 23\penalty0 (4):\penalty0
  2341--2368, 2013.

\bibitem[Goyal et~al.(2017)Goyal, Doll{\'a}r, Girshick, Noordhuis, Wesolowski,
  Kyrola, Tulloch, Jia, and He]{goyal2017accurate}
Goyal, P., Doll{\'a}r, P., Girshick, R., Noordhuis, P., Wesolowski, L., Kyrola,
  A., Tulloch, A., Jia, Y., and He, K.
\newblock Accurate, large minibatch sgd: Training imagenet in 1 hour.
\newblock \emph{arXiv preprint arXiv:1706.02677}, 2017.

\bibitem[Keskar \& W{\"a}chter(2019)Keskar and W{\"a}chter]{keskar2019limited}
Keskar, N. and W{\"a}chter, A.
\newblock A limited-memory quasi-newton algorithm for bound-constrained
  non-smooth optimization.
\newblock \emph{Optimization Methods and Software}, 34\penalty0 (1):\penalty0
  150--171, 2019.

\bibitem[Kingma \& Ba(2014)Kingma and Ba]{kingma2014adam}
Kingma, D.~P. and Ba, J.
\newblock Adam: A method for stochastic optimization.
\newblock In \emph{International Conference on Learning Representations}, 2014.

\bibitem[Liu \& Nocedal(1989)Liu and Nocedal]{Liu89onthe}
Liu, D.~C. and Nocedal, J.
\newblock On the limited memory bfgs method for large scale optimization.
\newblock \emph{MATHEMATICAL PROGRAMMING}, 45:\penalty0 503--528, 1989.

\bibitem[Loshchilov \& Hutter(2016)Loshchilov and Hutter]{loshchilov2016sgdr}
Loshchilov, I. and Hutter, F.
\newblock Sgdr: Stochastic gradient descent with warm restarts.
\newblock \emph{arXiv preprint arXiv:1608.03983}, 2016.

\bibitem[Ma(2020)]{ma2020apollo}
Ma, X.
\newblock Apollo: An adaptive parameter-wise diagonal quasi-newton method for
  nonconvex stochastic optimization.
\newblock \emph{arXiv preprint arXiv:2009.13586}, 2020.

\bibitem[McMahan \& Streeter(2010)McMahan and Streeter]{mcmahan2010adaptive}
McMahan, H.~B. and Streeter, M.~J.
\newblock Adaptive bound optimization for online convex optimization.
\newblock In \emph{Conference on Learning Theory}, pp.\  244--256, 2010.

\bibitem[Nesterov \& Polyak(2006)Nesterov and Polyak]{nesterov2006cubic}
Nesterov, Y. and Polyak, B.~T.
\newblock Cubic regularization of newton method and its global performance.
\newblock \emph{Mathematical Programming}, 108\penalty0 (1):\penalty0 177--205,
  2006.

\bibitem[Nguyen et~al.(2017)Nguyen, Liu, Scheinberg, and
  Tak{\'a}{\v{c}}]{nguyen2017sarah}
Nguyen, L.~M., Liu, J., Scheinberg, K., and Tak{\'a}{\v{c}}, M.
\newblock Sarah: A novel method for machine learning problems using stochastic
  recursive gradient.
\newblock In \emph{International Conference on Machine Learning}, pp.\
  2613--2621. PMLR, 2017.

\bibitem[Ott et~al.(2019)Ott, Edunov, Baevski, Fan, Gross, Ng, Grangier, and
  Auli]{ott2019fairseq}
Ott, M., Edunov, S., Baevski, A., Fan, A., Gross, S., Ng, N., Grangier, D., and
  Auli, M.
\newblock fairseq: A fast, extensible toolkit for sequence modeling.
\newblock In \emph{Proceedings of NAACL-HLT 2019: Demonstrations}, 2019.

\bibitem[Pan et~al.(2017)Pan, Innanen, and Liao]{pan2017accelerating}
Pan, W., Innanen, K.~A., and Liao, W.
\newblock Accelerating hessian-free gauss-newton full-waveform inversion via
  l-bfgs preconditioned conjugate-gradient algorithm.
\newblock \emph{Geophysics}, 82\penalty0 (2):\penalty0 R49--R64, 2017.

\bibitem[Papineni et~al.(2002)Papineni, Roukos, Ward, and
  Zhu]{papineni-etal-2002-bleu}
Papineni, K., Roukos, S., Ward, T., and Zhu, W.-J.
\newblock {B}leu: a method for automatic evaluation of machine translation.
\newblock In \emph{Proceedings of the 40th Annual Meeting of the Association
  for Computational Linguistics}, pp.\  311--318, Philadelphia, Pennsylvania,
  USA, July 2002. Association for Computational Linguistics.
\newblock \doi{10.3115/1073083.1073135}.
\newblock URL \url{https://www.aclweb.org/anthology/P02-1040}.

\bibitem[Paszke et~al.(2017)Paszke, Gross, Chintala, Chanan, Yang, DeVito, Lin,
  Desmaison, Antiga, and Lerer]{paszke2017automatic}
Paszke, A., Gross, S., Chintala, S., Chanan, G., Yang, E., DeVito, Z., Lin, Z.,
  Desmaison, A., Antiga, L., and Lerer, A.
\newblock Automatic differentiation in pytorch.
\newblock 2017.

\bibitem[Pearlmutter(1994)]{Pearlmutter94fastexact}
Pearlmutter, B.~A.
\newblock Fast exact multiplication by the hessian.
\newblock \emph{Neural Computation}, 6:\penalty0 147--160, 1994.

\bibitem[Reddi et~al.()Reddi, Kale, and Kumar]{reddi1904convergence}
Reddi, S., Kale, S., and Kumar, S.
\newblock On the convergence of adam and beyond. arxiv 2019.
\newblock \emph{arXiv preprint arXiv:1904.09237}.

\bibitem[Reddi et~al.(2018)Reddi, Kale, and Kumar]{reddi2018convergence}
Reddi, S.~J., Kale, S., and Kumar, S.
\newblock On the convergence of adam and beyond.
\newblock In \emph{International Conference on Learning Representations}, 2018.

\bibitem[Redmon(2013--2016)]{darknet13}
Redmon, J.
\newblock Darknet: Open source neural networks in c.
\newblock \url{http://pjreddie.com/darknet/}, 2013--2016.

\bibitem[Tran-Dinh et~al.(2019)Tran-Dinh, Pham, Phan, and
  Nguyen]{tran2019hybrid}
Tran-Dinh, Q., Pham, N.~H., Phan, D.~T., and Nguyen, L.~M.
\newblock A hybrid stochastic optimization framework for stochastic composite
  nonconvex optimization.
\newblock \emph{arXiv preprint arXiv:1907.03793}, 2019.

\bibitem[Tripuraneni et~al.(2017)Tripuraneni, Stern, Jin, Regier, and
  Jordan]{tripuraneni2017stochastic}
Tripuraneni, N., Stern, M., Jin, C., Regier, J., and Jordan, M.~I.
\newblock Stochastic cubic regularization for fast nonconvex optimization.
\newblock \emph{arXiv preprint arXiv:1711.02838}, 2017.

\bibitem[Xu et~al.(2020{\natexlab{a}})Xu, Roosta, and Mahoney]{xu2020newton}
Xu, P., Roosta, F., and Mahoney, M.~W.
\newblock Newton-type methods for non-convex optimization under inexact hessian
  information.
\newblock \emph{Mathematical Programming}, 184\penalty0 (1):\penalty0 35--70,
  2020{\natexlab{a}}.

\bibitem[Xu et~al.(2020{\natexlab{b}})Xu, Roosta, and Mahoney]{xu2020second}
Xu, P., Roosta, F., and Mahoney, M.~W.
\newblock Second-order optimization for non-convex machine learning: An
  empirical study.
\newblock In \emph{Proceedings of the 2020 SIAM International Conference on
  Data Mining}, pp.\  199--207. SIAM, 2020{\natexlab{b}}.

\bibitem[Yao et~al.(2020)Yao, Gholami, Shen, Keutzer, and
  Mahoney]{yao2020adaHessian}
Yao, Z., Gholami, A., Shen, S., Keutzer, K., and Mahoney, M.~W.
\newblock Adahessian: An adaptive second order optimizer for machine learning.
\newblock \emph{arXiv preprint arXiv:2006.00719}, 2020.

\bibitem[Zhou et~al.(2019)Zhou, Xu, and Gu]{JMLR:v20:19-055}
Zhou, D., Xu, P., and Gu, Q.
\newblock Stochastic variance-reduced cubic regularization methods.
\newblock \emph{Journal of Machine Learning Research}, 20\penalty0
  (134):\penalty0 1--47, 2019.
\newblock URL \url{http://jmlr.org/papers/v20/19-055.html}.

\bibitem[Zinkevich(2003)]{zinkevich2003online}
Zinkevich, M.
\newblock Online convex programming and generalized infinitesimal gradient
  ascent.
\newblock In \emph{Proceedings of the 20th international conference on machine
  learning (ICML)}, pp.\  928--936, 2003.

\end{thebibliography}
\bibliographystyle{icml2021}
}

\clearpage
\appendix

\section{Appendix}\label{appendix}
\subsection{License}
Image Classification: Imagenet has BSD 3-Clause License, Resnet has Apache License, Cifar10 has MIT License.

Neural Machine Translation: Fairseq has MIT License.

All experiments are implemented on Pytorch which has BSD License. Other assets that we use have no license.
\subsection{Additional details on the Experiments}
\textbf{Image Classification:} Here we provide some extra details of our experiments. From the results in Table \ref{sample-table}, we can see that SGDHess achieves the best accuracy among all optimizers. SGDHess also has the lowest standard deviation, indicating that it consistently performs well in all experiments. For Imagenet task, our code is based on the official implementation of Imagenet on Pytorch. We also keep all the default settings constant. The only thing that we change is the learning rate schedule (from step decay every 30 epochs to plateu decay where we decrease our learning rate by a factor of two if we do not make progress in three consecutive epochs) based on the suggestion from \citep{yao2020adaHessian}.
\begin{figure*}[ht!]
\centering
 \begin{subfigure}
     \centering
     \includegraphics[width=0.48\textwidth]{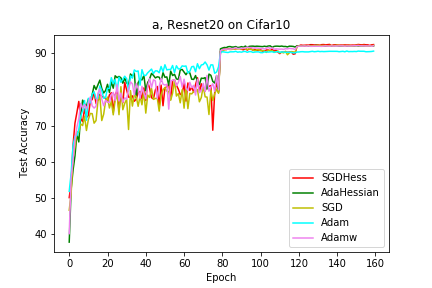}
     \label{fig:cifar}
    \quad
  \end{subfigure}
  \begin{subfigure}
     \centering
     \includegraphics[width=0.48\textwidth]{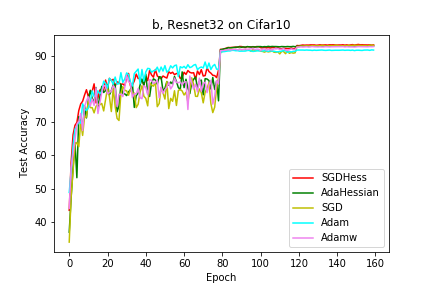}
     \label{fig:cifar}
    \quad
  \end{subfigure}
      \begin{subfigure}
     \centering
     \includegraphics[width=0.48\textwidth]{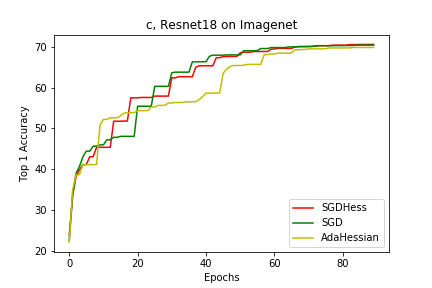}
     \label{fig:imagenet}
  \end{subfigure}
\caption{ (a) Test accuracy on Resnet20 (b) Test accuracy on Resnet32 (c) Test accuracy on Imagenet
    }
  \label{fig:experiments}
\end{figure*}
\begin{table}
  \caption{Test accuracy on Cifar10}
  \label{sample-table}
  \centering
  \begin{tabular}{lll}
    \toprule
    \cmidrule(r){1-2}
        & Resnet20    & Resnet32  \\
    \midrule
    SGD   & $92.14 \pm 0.15$ & $ 93.08 \pm 0.12$      \\
    AdaHessian     & $92.11 \pm 0.07$     & $92.96 \pm 0.09$ \\
     Adam    & $90.58 \pm 0.28$       & $91.62 \pm 0.12$ \\
     AdamW    & $92.05 \pm 0.15$      & $92.45 \pm 0.25$\\
     \midrule
     SGDHess & 92.46 $\pm 0.07$ & 93.19 $\pm$ 0.08     \\
    \bottomrule
  \end{tabular}
\end{table}

\textbf{Neural Machine Translation:} The settings of our experiments follow exactly the settings specified in the translation examples in \citep{ott2019fairseq}. The only things that we tune are learning rate, weight decay, and number of updates if needs be. Other than the main results reported in Table \ref{bleu}, we also run extra experiment with our adaptive algorithm described in Section \ref{adaptive}. The adaptive algorithm achieves the best BLEU score of 35.53, which is slightly worse than the non-adaptive algorithm. The main advantage of the adaptive algorithm is that it is a bit less sensitive to the change in the learning rate. For the non-adaptive algorithm, we need to warm up our optimizer very gradually (we set the number of updates to 8000) else we would run into exploding gradients problem. On the other hand, for the adaptive one, we can just set the default number of updates without any problems. We also suspect that incorporating the per-coordinate or diagonal-style adaptivity of popular optimizers such as Adam may provide an useful future direction for improvement.
\subsection{Supplemental Lemmas}
\begin{Lemma} \label{etabound}
Let $X = \min \{ \frac{\sqrt{A}}{\sqrt{B}}, \frac{A^{1/3}}{C^{1/3}} \}$. Then:
\begin{align*}
    \frac{A}{X} + BX + CX^2 \le 2\sqrt{AB} + 2A^{2/3}C^{1/3}
\end{align*}
\end{Lemma}
\begin{proof}
Bound the first term:
\begin{align*}
    \frac{A}{X} \le \max\{\sqrt{AB},A^{2/3}C^{1/3} \} \le \sqrt{AB} + A^{2/3}C^{1/3}
\end{align*}
Let us do some case work for the second term. If $X = \frac{\sqrt{A}}{\sqrt{B}}$ then $BX = \sqrt{AB}$. Otherwise, if $X = \frac{A^{1/3}}{C^{1/3}}$, then $\frac{A^{1/3}}{C^{1/3}} \le \frac{\sqrt{A}}{\sqrt{B}} \Rightarrow \frac{\sqrt{B}}{C^{1/3}} \le A^{1/6}$. Therefore,
\begin{align*}
    BX = \frac{BA^{1/3}}{C^{1/3}} &\le \sqrt{B}A^{1/6}A^{1/3}\\
    &\le \sqrt{AB}
\end{align*}
In either case, $BX \le \sqrt{AB}$. Now repeat the same arguments for $CX^2$ term, we would get $CX^2 \le A^{2/3}C^{1/3}$. Then we can combine the bounds to get the desired result.
\end{proof}
\begin{Lemma}\label{alphabound}
Let $X = \max \{ \frac{B^{2/3}}{A^{2/3}}, \frac{C^{6/5}}{A^{6/5}}, \frac{D^{4/3}}{A^{4/3}} \}$. Then:
\begin{align*}
    A\sqrt{X} + \frac{B}{X} + \frac{C}{X^{1/3}} + \frac{D}{X^{1/4}} \le 2(B^{1/3}A^{2/3} + C^{3/5}A^{2/5} + D^{2/3}A^{1/3})
\end{align*}
\end{Lemma}
\begin{proof}
The proof is almost the same as the proof of Lemma \ref{etabound}. First, we bound the first term:
\begin{align*}
    A\sqrt{X} \le \max \{ \frac{B^{1/3}}{A^{2/3}}, \frac{C^{3/5}}{A^{2/5}}, \frac{D^{2/3}}{A^{1/3}} \} \le B^{1/3}A^{2/3} + C^{3/5}A^{2/5} + D^{2/3}A^{1/3}
\end{align*}
Then we can do some case works for the other 3 terms, which would get us $\frac{B}{X} + \frac{C}{X^{1/3}} + \frac{D}{X^{1/4}} \le B^{1/3}A^{2/3} + C^{3/5}A^{2/5} + D^{2/3}A^{1/3}$. Now combining the bounds to get the desired results.
\end{proof}
\begin{Lemma}\label{normalizedbound}
With $\alpha = \min \{ \max \{ \frac{1}{T^{2/3}}, \frac{\Delta^{4/5}\rho^{2/5}}{T^{4/5}\sigma_G^{6/5}}, \frac{(2\Delta\sigma_H)^{2/3}}{T^{2/3}\sigma_G^{4/3}} \}, 1 \}$ and $\eta = \min \{ \frac{\sqrt{2\Delta}\alpha^{1/4}}{\sqrt{T(L\sqrt{\alpha}+4\sigma_H})}, \frac{(\Delta\alpha)^{1/3}}{(\rho T)^{1/3}} \}$, we have:
\begin{align*}
    \frac{3\Delta}{2\eta T} +\frac{3L\eta}{4} + \frac{3\sigma_G}{\alpha T} + \frac{3\eta^2\rho}{2\alpha} + \frac{3\eta \sigma_H}{\sqrt{\alpha}} + 3\sigma_G\sqrt{\alpha} &\le \frac{6\sigma_G+ 54^{1/3}(\Delta \sigma_H)^{1/3}\sigma_G^{1/3}}{T^{1/3}}+ \frac{6\sigma_G^{2/5}\Delta^{2/5}\rho^{1/5}}{T^{2/5}} \\
    &+ \frac{\sqrt{9\Delta L} + \sqrt{72\Delta\sigma_H}}{\sqrt{T}} 
    +\frac{6\Delta^{2/3}\rho^{1/3}}{T^{2/3}}+ \frac{\sqrt{18\Delta\sigma_H}}{T^{3/2}} + \frac{3\Delta^{2/3}\rho^{1/3}}{T^{5/3}}
\end{align*}
\end{Lemma}
\begin{proof}
Let $  F = \frac{3\Delta}{2\eta T} +\frac{3L\eta}{4} + \frac{3\sigma_G}{\alpha T} + \frac{3\eta^2\rho}{2\alpha} + \frac{3\eta \sigma_H}{\sqrt{\alpha}} + 3\sigma_G\sqrt{\alpha}$. Applying Lemma \ref{etabound} with $A = \frac{3\Delta}{2T}$, $B = \frac{3L}{4}+\frac{3\sigma_H}{\sqrt{\alpha}}$, and $C =\frac{3\rho}{2\alpha} $ and $\eta$ set to the value of $X$ specified by the Lemma:
\begin{align*}
     F &\le 3\sqrt{2}\sqrt{\frac{\Delta}{T}(\frac{L}{4}+\frac{\sigma_H}{\sqrt{\alpha}})} +\frac{3\Delta^{2/3}\rho^{1/3}}{T^{2/3}\alpha^{1/3}} + \frac{3\sigma_G}{\alpha T} + 3\sigma_G\sqrt{\alpha} 
\end{align*}
Use $\sqrt{a+b} \le \sqrt{a} + \sqrt{b}$:
\begin{align}
    F \le \frac{3}{\sqrt{2}}\frac{\sqrt{\Delta L}}{\sqrt{T}} + 3\sqrt{2} \frac{\sqrt{\Delta\sigma_H}}{\sqrt{T}\alpha^{1/4}} + \frac{3\Delta^{2/3}\rho^{1/3}}{T^{2/3}\alpha^{1/3}} + \frac{3\sigma_G}{\alpha T} + 3\sigma_G\sqrt{\alpha} \label{maineqn}
\end{align}
Now applying Lemma \ref{alphabound} with $X = \alpha$, $A = 3\sigma_G$, $B = \frac{3\sigma_G}{T}$, $C= \frac{3\Delta^{2/3}\rho^{1/3}}{T^{2/3}}$, and $D =3\sqrt{2} \frac{\sqrt{\Delta\sigma_H}}{\sqrt{T}} $:
\begin{align*}
    F \le \frac{3\sqrt{\Delta L}}{\sqrt{2}\sqrt{T}} + \frac{6\sigma_G}{T^{1/3}} + \frac{6\sigma_G^{2/5}\Delta^{2/5}\rho^{1/5}}{T^{2/5}}+\frac{54^{1/3}(\Delta \sigma_H)^{1/3}\sigma_G^{1/3}}{T^{1/3}}
\end{align*}
Let $\alpha = \min \{ \max \{ \frac{1}{T^{2/3}}, \frac{\Delta^{4/5}\rho^{2/5}}{T^{4/5}\sigma_G^{6/5}}, \frac{(2\Delta\sigma_H)^{2/3}}{T^{2/3}\sigma_G^{4/3}} \}, 1 \}$. Since $\frac{1}{T^{2/3}} \le 1$, let us examine the other two cases when $\alpha \ge 1$.

\emph{Case 1: $\frac{\Delta^{4/5}\rho^{2/5}}{T^{4/5}G^{6/5}} \ge 1$}. Then we have $\sigma_G\le \frac{\Delta^{2/3}\rho^{1/3}}{T^{2/3}}$. Substitute to (\ref{maineqn}) with $\alpha = 1$:
\begin{align*}
    F &\le \frac{3\sqrt{\Delta L}}{\sqrt{2}\sqrt{T}} +\frac{3\sqrt{2}\sqrt{\Delta\sigma_H}}{\sqrt{T}} + \frac{3\Delta^{2/3}\rho^{1/3}}{T^{2/3}} + \frac{3\sigma_G}{T} + 3\sigma_G \\
    &\le \frac{3\sqrt{\Delta L}}{\sqrt{2}\sqrt{T}} +\frac{3\sqrt{2}\sqrt{\Delta\sigma_H}}{\sqrt{T}} + \frac{3\Delta^{2/3}\rho^{1/3}}{T^{2/3}} + \frac{3\Delta^{2/3}\rho^{1/3}}{T^{5/3}} + \frac{3\Delta^{2/3}\rho^{1/3}}{T^{2/3}}
\end{align*}
\emph{Case 2: $\frac{(2\Delta\sigma_H)^{2/3}}{T^{2/3}\sigma_G^{4/3}} \ge 1$}. Then we have $\sigma_G \le \frac{\sqrt{2\Delta \sigma_H}}{\sqrt{T}}$. Substitute to (\ref{maineqn}) with $\alpha = 1$:
\begin{align*}
    F &\le \frac{3\sqrt{\Delta L}}{\sqrt{2}\sqrt{T}} +\frac{3\sqrt{2}\sqrt{\Delta\sigma_H}}{\sqrt{T}} + \frac{3\Delta^{2/3}\rho^{1/3}}{T^{2/3}} + \frac{3\sqrt{2\Delta\sigma_H}}{T^{3/2}} + \frac{3\sqrt{2\Delta \sigma_H}}{\sqrt{T}} 
\end{align*}
Now combine the bounds to get the desired result.
\end{proof}
\begin{Lemma} \label{errorLemma}
Define:
\begin{align*}
    \hat \epsilon_t = \hat g_t^{clip} - \nabla F(\bx_t)
\end{align*}
Now, for some constant $\sigma_H$, set $K = \frac{2G^2\rho^2}{-2\sigma_H^2+\sqrt{4\sigma_H^4+\frac{\rho^2G^2}{2}}}$, $\eta_t = \frac{c}{(w+ \sum_{i=1}^{t-2} G_i^2)^{1/3}}$ with $c \le \frac{2G^{2/3}}{\sqrt{K}}$ and $w = \max\{(4Lc)^3, 3G^2\}$, $\alpha_t = 2K\eta_t\eta_{t+1}$. Then we have:
\begin{align*}
     \sum_{t=1}^T \E\left[\frac{6}{5K\eta_{t+1}}\|\hat \epsilon_{t+1}\|^2 -\frac{6}{5K\eta_{t}}\|\hat \epsilon_t\|^2 \right] \le \sum_{t=1}^TE\left[-\frac{3\eta_t}{4}\|\hat\epsilon_t\|^2 + \frac{\eta_t}{5} \|\nabla F(\bx_t)\|^2 \right]&+ \frac{24}{5}Kc^2\ln(T+1)\\
     &+  \frac{16G^4}{25K^2c^3} \ln T
\end{align*}
\end{Lemma}
\begin{proof}
Similar to Lemma \ref{thm:sgderror}, let us define:
\begin{align*}
    \epsilon^G_t &= \nabla f(\bx_t,z_t) - \nabla F(\bx_t)\\
    \epsilon^H_t &=\nabla^2 f(\bx_{t+1}, z_{t+1})(\bx_{t+1} -\bx_t) -  \nabla^2 F(\bx_{t+1})(\bx_{t+1} -\bx_t)
\end{align*}
Note that we have the important properties:
\begin{align*}
    \E[\epsilon^G_{t}|z_1, z_2, ..., z_{t-1}]&=0\\
    \E[\epsilon^H_t|z_1, z_2,..., z_t]&=0
\end{align*}
Further, we have:
\begin{align*}
    \E[\|\epsilon^G_t\|^2]
    &\le \E[G_t^2]
\end{align*}
And:
\begin{align*}
    \E[\|\epsilon^H_t\|^2]&\le \E[\|\nabla^2 f(\bx_{t+1}, z_{t+1})(\bx_{t+1} -\bx_t) -  \nabla^2 F(\bx_{t+1})(\bx_{t+1} -\bx_t)\|^2]\\&\le \sigma_H^2\|\bx_{t+1}-\bx_t\|^2
\end{align*}
Also,
\begin{align*}
    \E[\|\hat\epsilon_t\|^2]&\le \E[\|\hat g^{clip}_{t}-\nabla F(\bx_t)\|^2]\\&\le 4G^2
\end{align*}
Finally, also note that we must have:
\begin{align*}
    \|\hat g^{clip}_{t}\|\le G
\end{align*} 
for all $t$ due to our definition of $\hat g^{clip}_{t}$. \\
Let us define another quantity:
\begin{align*}
    \hat \epsilon_{t+1}^{\text{noclip}}=\hat g_{t+1}-\nabla F(\bx_{t+1})
\end{align*}
Now, we derive a recursive formula for $\hat \epsilon_{t+1}^{\text{noclip}}$ in terms of $\hat \epsilon_t$:
\begin{align*}
    \hat \epsilon_{t+1}^{\text{noclip}} &= \hat g_{t+1} - \nabla F(\bx_{t+1})\nonumber\\
    &=(1-\alpha_t)(\hat g^{clip}_{t}+\nabla^2 f(\bx_{t+1}, z_{t+1})(\bx_{t+1} - \bx_t)) +\alpha_t \nabla f(\bx_{t+1}, z_{t+1}) - \nabla F(\bx_{t+1})\nonumber\\
    &=(1-\alpha_t)(\hat \hat g_{tclip} +\nabla^2 f(\bx_{t+1}, z_{t+1})(\bx_{t+1} - \bx_{t})- \nabla F(\bx_{t+1})) + \alpha_t(\nabla f(\bx_{t+1}, z_{t+1}) - \nabla F(\bx_{t+1}))\nonumber\\
    &=(1-\alpha_t)(\hat g^{clip}_{t} - \nabla F(\bx_t)) + (1-\alpha_t)(\nabla^2 f(\bx_{t+1}, z_{t+1})-\nabla^2 F(\bx_{t+1})(\bx_{t+1}-\bx_t))\nonumber\\
&\quad\quad+
(1-\alpha_t)(\nabla F(\bx_t) +  \nabla^2F(\bx_{t+1})(\bx_{t+1}-\bx_t) -\nabla F(\bx_{t+1})) + \alpha_t \epsilon^G_{t+1}
\end{align*}
From the analysis of Lemma \ref{thm:sgderror}, we have the relation:
\begin{align}
    \|\hat \epsilon_{t+1}\| \le \|\hat \epsilon_{t+1}^{noclip}\| \label{eqn:clipvsnoclip}
\end{align}
And we also have:
\begin{align}
    \hat\epsilon_{t+1}^{\text{noclip}} &= (1-\alpha_t)(\hat g^{clip}_{t} - \nabla F(\bx_t)) + (1-\alpha_t)(\nabla^2 f(\bx_{t+1}, z_{t+1})-\nabla^2 F(\bx_{t+1})(\bx_{t+1}-\bx_t))\nonumber\\
&\quad\quad+
(1-\alpha_t)(\nabla F(\bx_t) +  \nabla^2F(\bx_{t+1})(\bx_{t+1}-\bx_t) -\nabla F(\bx_{t+1})) + \alpha_t \epsilon^G_{t+1}\label{eqn:epsrecursion1}
\end{align}
Let:
\begin{align}
    \delta_t &= \nabla F(\bx_t) +  \nabla^2F(\bx_{t+1})(\bx_{t+1}-\bx_t) -\nabla F(\bx_{t+1})\nonumber\\
    \|\delta_t\|^2 &\le \frac{\rho^2}{4}\|\bx_{t+1}-\bx_t\|^4 \label{eqn:deltat}
\end{align}
Equation (\ref{eqn:epsrecursion1}) becomes:
\begin{align*}
    \hat \epsilon_{t+1}^{\text{noclip}} &= (1-\alpha_t)\hat{\epsilon_t} +(1-\alpha_t)\epsilon^H_t
    +(1-\alpha_t)\delta_t+\alpha_t \epsilon^G_{t+1}
\end{align*}
Now use relation (\ref{eqn:clipvsnoclip}), we have:
\begin{align*}
    \|\hat \epsilon_{t+1}\|^2 \le \|(1-\alpha_t)\hat{\epsilon_t} +(1-\alpha_t)\epsilon^H_t
    +(1-\alpha_t)\delta_t+\alpha_t \epsilon^G_{t+1} \|^2
\end{align*}
Multiply $\frac{12\eta_t}{5\eta_t}$ to both sides and take the expectation of the above equation:
\begin{align*}
 \E[\frac{12\eta_t}{5\alpha_t}\|\hat \epsilon_{t+1}\|^2] &\le \E[\frac{12\eta_t}{5\alpha_t}\|(1-\alpha_t)\hat{\epsilon_t} +(1-\alpha_t)\epsilon^H_t
    +(1-\alpha_t)\delta_t+\alpha_t \epsilon^G_{t+1} \|^2] 
\end{align*}
Notice that by definition $\frac{\eta_t}{\alpha_t} = \frac{\eta_t}{2K\eta_t\eta_{t+1}}= \frac{1}{2K\eta_{t+1}} = \frac{2K(w+\sum_{i=1}^{t-1}G_i^2)^{1/3}}{c}$ which is independent of the current sample $z_t$. Thus when we take expectation with respect to sample $z_t$, we can consider $\frac{\eta_t}{\alpha_t}$ as a constant. For example, let us analyze $\E\left[\frac{\eta_t}{\alpha_t}\langle \hat \epsilon_t, \epsilon_t^H \rangle\right]$:
\begin{align*}
    \E_{z_1,...,z_t}\left[\frac{\eta_t}{\alpha_t}\langle \hat \epsilon_t, \epsilon_t^H \rangle \rangle \right] &= \E_{z_1,...,z_{t-1}}\left[\E_{z_t}\left[\frac{\eta_t}{\alpha_t}\langle \hat \epsilon_t , \epsilon_t^H \rangle | z_1,..., z_{t-1}\right]\right] \\
    &= \E_{z_1,...,z_{t-1}}\left[\frac{\eta_t}{\alpha_t}\E_{ z_t}\left[\langle \hat \epsilon_t , \epsilon_t^H \rangle | z_1,..., z_{t-1}\right]\right] 
\end{align*}
Then the cross-terms $E[\langle \hat \epsilon_t, \epsilon_t^H \rangle ], E[\langle \delta_t, \epsilon_t^H \rangle ], E[\langle \epsilon_{t+1}^G, \epsilon_t^H \rangle ],  E[\langle \epsilon_{t+1}^G, \hat \epsilon_t \rangle ],  E[\langle \epsilon_{t+1}^G, \delta_t \rangle ]$ all become zero in expectation. Then: 
\begin{align*}
    \E[\frac{12\eta_t}{\alpha_t}\|\hat\epsilon_{t+1}\|^2] &\le \E[\frac{12\eta_t}{5\alpha_t}\left[(1-\alpha_t)^2(\|\hat \epsilon_t\|^2 + \|\epsilon^H_t\|^2 + \|\delta_t\|^2) + \alpha_t^2 \|\epsilon^G_{t+1}\|^2 + 2(1-\alpha_t)^2\langle \hat \epsilon_t, \delta_t\rangle\right]]
\end{align*}
Applying Young's inequality, for any $\lambda$ we have:
\begin{align}
    \langle\hat{\epsilon_{t}}, \delta_t\rangle \leq \frac{\lambda\|\hat{\epsilon_t}\|^2}{2} + \frac{\|\delta_t\|^2}{2\lambda} \label{eqn: crosstermerror}
\end{align}
Using (\ref{eqn:deltat}) and (\ref{eqn: crosstermerror}), we have:
\begin{align*}
    \E[\frac{12\eta_t}{\alpha_t}\|\hat\epsilon_{t+1}\|^2]&\le \E[\frac{12\eta_t}{\alpha_t}[(1-\alpha_t)^2\|\hat{\epsilon_{t}}\|^2 + (1-\alpha_t)^2\sigma_H^2\|\bx_{t+1} - \bx_{t}\|^2 + (1-\alpha_t)^2\frac{\rho^2}{4}\|\bx_{t+1}-\bx_t\|^4+ \alpha_t^2G_t^2\\
&\quad\quad+ (1-\alpha_t)^2(\lambda \|\hat{\epsilon_t}\|^2 + \frac{\|\delta_t\|^2}{\lambda} )]]\\
    &\le\E[\frac{12\eta_t}{\alpha_t}[(1-\alpha_t)^2\|\hat{\epsilon_{t}}\|^2 + (1-\alpha_t)^2\sigma_H^2\|\bx_{t+1} - \bx_{t}\|^2 + (1-\alpha_t)^2\frac{\rho^2}{4}\|\bx_{t+1}-\bx_t\|^4 + \alpha_t^2G_t^2 \\
&\quad\quad+ (1-\alpha_t)^2(\lambda \|\hat{\epsilon_t}\|^2) + (1-\alpha_t)^2 \frac{\rho^2}{4\lambda} \|\bx_{t+1}-\bx_t\|^4]] 
\end{align*}
Next, we observe:
\begin{align*}
    \|\bx_t-\bx_{t+1}\|&\le \eta_t \|\hat g^{clip}_{t}\|\\
    &\le \eta_t(\|\nabla F(\bx_t)\| + \|\hat\epsilon_t\|)
\end{align*}
and:
\begin{align*}
    \|\hat g^{clip}_{t}\|^4 \leq G^2\|\hat g^{clip}_{t}\|^2
\end{align*}
So plugging this back in yields:
\begin{align*}
    \E[\frac{12\eta_t}{\alpha_t}\|\hat\epsilon_{t+1}\|^2]&\le \E[\frac{12\eta_t}{\alpha_t}[(1-\alpha_t)^2[\|\hat{\epsilon_{t}}\|^2 + (1-\alpha_t)^2\sigma_H^2
\eta_{t}^2(\|\nabla F(\bx_t)\|^2 + 2\langle\|\nabla F(\bx_t)\|, \|\hat\epsilon_t\|\rangle+  \|\hat{\epsilon_t}\|^2)\\
&\quad\quad+ \alpha_t^2G_t^2 +(1-\alpha_t)^2\frac{\rho^2}{4}\eta_{t}^4G^2(\|\nabla F(\bx_t)\|^2+ 2\langle\|\nabla F(\bx_t)\|, \|\hat\epsilon_t\|\rangle+  \|\hat{\epsilon_t}\|^2) + (1-\alpha_t)^2(\lambda \|\hat{\epsilon_t}\|^2)\\
&\quad\quad+ (1-\alpha_t)^2 \frac{\rho^2}{4\lambda} \eta_{t}^4G^2(||\nabla F(\bx_t)||^2 + 2\langle\|\nabla F(\bx_t)\|, \|\hat\epsilon_t\|\rangle+  \|\hat{\epsilon_t}||^2)]]
\end{align*}
Again applying Young's Inequality with $\lambda = 1$:
\begin{align*}
    \E[\frac{12\eta_t}{\alpha_t}\|\hat\epsilon_{t+1}\|^2]&\le \E[\frac{12\eta_t}{\alpha_t}[(1-\alpha_t)^2\|\hat{\epsilon_{t}}\|^2] + (1-\alpha_t)^2H^2
\eta_{t}^2(2\|\nabla F(\bx_t)\|^2 + 2\|\hat{\epsilon_t}\|^2) \\
&\quad\quad+(1-\alpha_t)^2\frac{\rho^2}{4}\eta_{t}^4G^2(2\|\nabla F(\bx_t)\|^2+ 2 \|\hat{\epsilon_t}\|^2)+ \alpha_t^2G_t^2 \\
&\quad\quad + (1-\alpha_t)^2(\lambda \|\hat{\epsilon_t}\|^2)+ (1-\alpha_t)^2 \frac{\rho^2}{4\lambda} \eta_{t}^4G^2(2\|\nabla F(\bx_t)\|^2+  2\|\hat{\epsilon_t}\|^2)]]
\end{align*}
Since $ (1-\alpha_t)^2 \le 1 $:
\begin{align*}
     \E[\frac{12\eta_t}{\alpha_t}\|\hat\epsilon_{t+1}\|^2]&\le \E[\frac{12\eta_t}{\alpha_t}[(1-\alpha_t)^2\|\hat{\epsilon_{t}}\|^2] + 2\sigma_H^2
\eta_{t}^2(||\nabla F(\bx_t)||^2 + \|\hat{\epsilon_t}\|^2) + \alpha_t^2G_t^2 +\frac{\rho^2}{2}\eta_{t}^4G^2(\|\nabla F(x_t)\|^2 \\
&\quad\quad+
\|\hat{\epsilon_t}\|^2) + (\lambda \|\hat{\epsilon_t}\|^2) +  \frac{\rho^2}{2\lambda} \eta_{t}^4G^2(\|\nabla F(\bx_t)\|^2 + \|\hat{\epsilon_t}\|^2)]]\\
&= E[\frac{12\eta_t}{\alpha_t}[\|\hat \epsilon_t\|^2((1-\alpha_t)^2 +2\sigma_H^2\eta_{t}^2 + \frac{\rho^2}{2}\eta_{t}^4G^2 + \lambda + \frac{\rho^2}{2\lambda} \eta_{t}^4G^2 ) + \|\nabla F(\bx_t)\|^2 (2\sigma_H^2\eta_{t}^2 + \frac{\rho^2}{2}\eta_{t}^4G^2 \\
&\quad\quad+  \lambda + \frac{\rho^2}{2\lambda} \eta_{t}^4G^2) + \alpha_t^2G_t^2]]
\end{align*}
Now we want the coefficient of the error to be something like 1-$O(\alpha_t)$:
\begin{align*}
    (1-\alpha_t)^2 +2\sigma_H^2\eta_{t}^2 + \frac{\rho^2}{2}\eta_{t}^4G^2 + \lambda + \frac{\rho^2}{2\lambda} \eta_{t}^4G^2 \le 1 - \frac{5}{12}\alpha_t \ (*) 
\end{align*}
Let $$ \alpha_t \leq 1$$ and $$ \lambda = \frac{\alpha_t}{2}$$
For (*) to be satisfied: 
\begin{align*}
    \eta_{t}^4 (\frac{\rho^2}{2}G^2 + \frac{\rho^2}{2\lambda}G^2) + 2\sigma_H^2\eta_{t}^2 - \frac{\alpha_t}{12} \le 0 
\end{align*}
Solving the quadratic equation, we get:
\begin{align*}
    \eta_{t}^2 &\le \frac{-2\sigma_H^2 +  \sqrt{4\sigma_H^4+\frac{\rho^2 G^2\alpha_t}{6}+\frac{\rho^2G^2}{3}}}{G^2\rho^2+\frac{2G^2\rho^2}{\alpha_t}}\\
    &\le \frac{-2\sigma_H^2 +  \sqrt{4\sigma_H^4+\frac{\rho^2G^2}{6}+\frac{\rho^2G^2}{3}}}{G^2\rho^2+\frac{2G^2\rho^2}{\alpha_t}}  \\
    &\le \frac{-2\sigma_H^2+\sqrt{4\sigma_H^4+\frac{\rho^2G^2}{2}}}{G^2\rho^2}\frac{\alpha_t}{2}
\end{align*}
Let  $K = \frac{2G^2\rho^2}{-2\sigma_H^2+\sqrt{4\sigma_H^4+\frac{\rho^2G^2}{2}}}$:
\begin{align}
    \eta_{t}^2K \le \alpha_t \label{eqn: momentine}
\end{align}
So overall we get:
\begin{align*}
    \E[\frac{12\eta_t}{5\alpha_t}\|\hat \epsilon_{t+1}\|^2] &\le \E\left[\frac{12\eta_t}{5\alpha_t} [(1-\frac{5}{12}\alpha_t)\|\hat{\epsilon_{t}}\|^2 + \frac{1}{12}\alpha_{t} \|\nabla F(\bx_t)\|^2 + \alpha_t^2G_t^2]\right]
\end{align*}
Let $\eta_t = \frac{c}{(w+\sum_{i=1}^{t-2} G_i^2)^{1/3}}$ and $\alpha_t = 2K\eta_t\eta_{t+1}$. Then:
\begin{align*}
    \E[\frac{12\eta_t}{5\alpha_t}\|\hat \epsilon_{t+1}\|^2] &\le \E\left[\frac{12\eta_t}{5\alpha_t} [(1-\frac{5}{12}\alpha_t)\|\hat{\epsilon_{t}}\|^2 + \frac{1}{12}\alpha_{t} \|\nabla F(\bx_t)\|^2 + \alpha_t^2G_t^2]\right]\\
    &= \E\left[\frac{12\eta_t}{5\alpha_t}\|\hat \epsilon_t\|^2 -\eta_t\|\hat\epsilon_t\|^2 + \frac{\eta_t}{5}\|\nabla F(\bx_t)\|^2 + \frac{12}{5}\eta_t\alpha_tG_t^2\right] \\
    \Rightarrow  \E[\frac{12\eta_t}{5\alpha_t}\|\hat \epsilon_{t+1}\|^2 -\frac{12\eta_t}{5\alpha_t}\|\hat \epsilon_t\|^2 ] &\le \E\left[ -\eta_t\|\hat\epsilon_t\|^2 + \frac{\eta_t}{5}\|\nabla F(\bx_t)\|^2 + \frac{12}{5}\eta_t\alpha_tG_t^2\right] \\
    \Leftrightarrow \E[\frac{6}{5K\eta_{t+1}}\|\hat \epsilon_{t+1}\|^2 -\frac{6}{5K\eta_{t+1}}\|\hat \epsilon_t\|^2 ] &\le \E\left[ -\eta_t \|\hat{\epsilon_{t}}\|^2 + \frac{\eta_t}{5} \|\nabla F(\bx_t)\|^2 + \frac{24}{5}\eta_t^3KG_t^2\right]
\end{align*}
Subtract $E[\frac{6}{5K\eta_t}\|\hat \epsilon_t\|^2]$ from both sides:
\begin{align*}
    \E[\frac{6}{5K\eta_{t+1}}\|\hat \epsilon_{t+1}\|^2 -\frac{6}{5K\eta_{t}}\|\hat \epsilon_t\|^2 ] &\le \E\left[\frac{6}{5K}(\frac{1}{\eta_{t+1}} - \frac{1}{\eta_t})\|\hat{\epsilon_{t}}\|^2 -\eta_t \|\hat{\epsilon_{t}}\|^2 + \frac{\eta_t}{5} \|\nabla F(\bx_t)\|^2 + \frac{24}{5}\eta_t^3KG_t^2\right]
\end{align*}
Now, let us analyze $\frac{1}{\eta_{t+1}} - \frac{1}{\eta_t}$:
\begin{align*}
    \frac{1}{\eta_{t+1}} - \frac{1}{\eta_t} &= \frac{1}{c}\left[(w+\sum_{i=1}^{t-1}G_i^2)^{1/3} - (w+\sum_{i=1}^{t-2}G_i^2)^{1/3}\right] \\
    &\le \frac{G_{t-1}^2}{3c(w+\sum_{i=1}^{t-2}G_i^2)^{2/3}} \\
    &= \frac{k^2G_{t-1}^2}{3c^3(w+\sum_{i=1}^{t-2}G_i^2)^{2/3}} \\ 
    &= \frac{\eta_t^2G_{t-1}^2}{3c^3} \\
    &\le \frac{G_{t-1}^2\lambda\eta_t}{6c^3} + \frac{G_{t-1}^2\eta_t^3}{6c^3\lambda}
\end{align*}
Plug in:
\begin{align*}
     \E[\frac{6}{5K\eta_{t+1}}\|\hat \epsilon_{t+1}\|^2 -\frac{6}{5K\eta_{t}}\|\hat \epsilon_t\|^2 ] &\le \E\left[\frac{6}{5K}(\frac{G_{t-1}^2\lambda\eta_t}{6c^3} + \frac{G_{t-1}^2\eta_t^3}{6c^3\lambda})\|\hat{\epsilon_{t}}\|^2 -\eta_t \|\hat{\epsilon_{t}}\|^2 + \frac{\eta_t}{5} \|\nabla F(\bx_t)\|^2 + \frac{24}{5}\eta_t^3KG_t^2\right]\\
     &= \E\left[-\eta_t\|\hat{\epsilon_{t}}\|^2 \left(1- \frac{G_{t-1}^2\lambda}{5Kc^3}\right)+ \frac{\eta_t}{5} \|\nabla F(\bx_t)\|^2 + \eta_t^3 \left(\frac{24}{5}KG_t^2 + \frac{G_{t-1}^2\|\hat \epsilon_t\|^2}{5Kc^3\lambda}\right)\right]
\intertext{Let $\lambda = \frac{5Kc^3}{4G_{t-1}^2}$ and use the fact that $\E[\|\hat \epsilon\|^2] \le 4G^2$:}
&\le \E\left[-\frac{3\eta_t}{4}\|\hat\epsilon_t\|^2 + \frac{\eta_t}{5} \|\nabla F(\bx_t)\|^2 +\eta_t^3 \left(\frac{24}{5}KG_t^2 + \frac{16G^2G_{t-1}^4}{25K^2c^6}\right) \right] \\
&\le \E\left[-\frac{3\eta_t}{4}\|\hat\epsilon_t\|^2 + \frac{\eta_t}{5} \|\nabla F(\bx_t)\|^2 +\eta_t^3 \left(\frac{24}{5}KG_t^2 + \frac{16G^4G_{t-1}^2}{25K^2c^6}\right) \right]
\end{align*}
Now sum over t:
\begin{align*}
    \sum_{t=1}^T  \E[\frac{6}{5K\eta_{t+1}}\|\hat \epsilon_{t+1}\|^2 -\frac{6}{5K\eta_{t}}\|\hat \epsilon_t\|^2 ] &\le \sum_{t=1}^T \E\left[-\frac{3\eta_t}{4}\|\hat\epsilon_t\|^2 + \frac{\eta_t}{5} \|\nabla F(\bx_t)\|^2 +\eta_t^3 \left(\frac{24}{5}KG_t^2 + \frac{16G^4G_{t-1}^2}{25K^2c^6}\right) \right]
\end{align*}
From Lemma 4 of \citep{cutkosky2019momentum}, we have the following:
\begin{align*}
    \sum_{t=1}^T \frac{a_t}{a_0 + \sum_{i=1}^t a_i} \le \ln\left(1+\frac{\sum_{i=1}^t a_i}{a_0}\right)
\end{align*}
Analyze third term:
\begin{align*}
    \sum_{t=1}^T \frac{24}{5}K\eta_t^3G_t^2 &= \sum_{t=1}^T \frac{24}{5}Kc^3 \frac{G_t^2}{w+\sum_{i=1}^{t-2}G_i^2}
\end{align*}
Now with $w \ge 3G^2 \ge G^2 + G_{t-1}^2 + G_t^2$, we would get:
\begin{align*}
    \sum_{t=1}^T \frac{24}{5}K\eta_t^3G_t^2 &\le \sum_{t=1}^T \frac{24}{5}Kc^3 \frac{G_t^2}{G^2+\sum_{i=1}^{t}G_i^2}\\
    &\le \frac{24}{5}Kc^2 \ln\left(1+\frac{\sum_{i=1}^TG_i^2}{G^2}\right) \\
    &\le \frac{24}{5}Kc^2\ln(T+1)
\end{align*}
Analyze the fourth term:
\begin{align*}
    \sum_{t=1}^T \frac{16G^4G_{t-1}^2\eta_t^3}{25K^2c^6} &= \sum_{t=1}^T\frac{16G^4}{25K^2c^3}\frac{G_{t-1}^2}{w+\sum_{i=1}^{t-2}G_i^2} \\
    &\le\sum_{t=1}^T\frac{16G^4}{25K^2c^3}\frac{G_{t-1}^2}{2G^2+\sum_{i=1}^{t-1}G_i^2}\\
    &\le \frac{16G^4}{25K^2c^3} \ln\left(1+\sum_{i=1}^{T-1}\frac{G_i^2}{G^2}\right)
    \\
    &\le \frac{16G^4}{25K^2c^3} \ln T
\end{align*}
Then:
\begin{align*}
    \sum_{t=1}^T \E[\frac{6}{5K\eta_{t+1}}\|\hat \epsilon_{t+1}\|^2 -\frac{6}{5K\eta_{t}}\|\hat \epsilon_t\|^2 ] &\le \sum_{t=1}^TE\left[-\frac{3\eta_t}{4}\|\hat\epsilon_t\|^2 + \frac{\eta_t}{5} \|\nabla F(\bx_t)\|^2 \right] + \frac{24}{5}Kc^2\ln(T+1)+  \frac{16G^4}{25K^2c^3} \ln T
\end{align*}
\end{proof}
\subsection{Proof of Lemma \ref{thm:onestep}}
\label{sec:proof2}
\onestep*
\begin{proof}
\begin{align*}
    F(\bx_{t+1})&\le F(\bx_t) +\langle \nabla F(\bx_t), \bx_{t+1}- \bx_t\rangle + \frac{L}{2}\|\bx_{t+1}-\bx_t\|^2\\
    &= F(\bx_t) -\eta_t \langle \nabla F(\bx_t), \hat g^{clip}_{t}\rangle  + \frac{\eta_t^2 L\|\hat g^{clip}_{t}\|^2}{2}
    \intertext{Taking expectation of both sides:}
    \E[F(\bx_{t+1})]&\le \E[F(\bx_t)] - \eta_t \E[\langle \nabla F(\bx_t), \hat g^{clip}_{t}\rangle] + \frac{\eta_t^2 L\E[\|\hat g^{clip}_{t}\|^2]}{2}\\
    &\le \E[F(\bx_t)] - \eta_t\E[\|\nabla F(\bx_t)\|^2]-\eta \E[\langle \nabla F(\bx_t),\hat\epsilon_t\rangle] + \frac{\eta_t^2 L\E[\|\hat g^{clip}_{t}\|^2]}{2}
    \intertext{Using Young's inequality:}
    &\le \E[F(\bx_t)] - \eta_t\E[\|\nabla F(\bx_t)\|^2]+\frac{\eta_t}{2} \E[\|\nabla F(\bx_t)\|^2] + \frac{\eta_t}{2}\E[\|\hat\epsilon_t\|^2] + \frac{\eta_t^2 L\E[\|\hat g^{clip}_t\|^2]}{2}\\
    &\le \E[F(\bx_t)] - \frac{\eta_t}{2}\E[\|\nabla F(\bx_t)\|^2] + \frac{\eta_t}{2}\E[\|\hat\epsilon_t\|^2] + \frac{\eta_t^2 L\E[\|\nabla F(\bx_t) + \hat\epsilon_t\|^2]}{2}
    \intertext{Using $\|a+b\|^2\le 2\|a\|^2+2\|b\|^2$:}
    &\le \E[F(\bx_t)] - \frac{\eta_t}{2}\E[\|\nabla F(\bx_t)\|^2] + \frac{\eta_t}{2}\E[\|\hat\epsilon_t\|^2] + \frac{\eta_t^2 L\E[2\|\nabla F(\bx_t)\| + 2\|\hat\epsilon_t\|^2]}{2}
    \intertext{Using $\eta_t \le \frac{1}{4L}$:}
    &\le \E[F(\bx_t)] - \frac{\eta_t}{2}\E[\|\nabla F(\bx_t)\|^2] + \frac{\eta_t}{2}\E[\|\hat\epsilon_t\|^2] + \frac{\eta_t\E[\|\nabla F(\bx_t)\| + \|\hat\epsilon_t\|^2]}{4}\\
    &= \E[F(\bx_t)]-\frac{\eta_t}{4} \E[\|\nabla F(\bx_t)\|^2]+ \frac{3\eta_t}{4}\E[\|\hat\epsilon_t\|^2]
\end{align*}
\end{proof}

\subsection{Proof of Lemma \ref{thm:sgderror}}
\label{sec:proof3}
\sgderror*
\begin{proof}
Let us additionally define:
\begin{align*}
    \epsilon^G_t &= \nabla f(\bx_t,z_t) - \nabla F(\bx_t)\\
    \epsilon^H_t &=\nabla^2 f(\bx_{t+1}, z_{t+1})(\bx_{t+1} -\bx_t) -  \nabla^2 F(\bx_{t+1})(\bx_{t+1} -\bx_t)
\end{align*}
Note that we have the important properties:
\begin{align*}
    \E[\epsilon^G_{t}]&=0\\
    \E[\epsilon^H_t]&=0
\end{align*}
Note however that $\E[\hat\epsilon_t]\ne 0$. Further, we have:
\begin{align*}
    \E[\|\epsilon^G_t\|^2]&= \E[\|\nabla f(\bx_t,z_t)\|^2 -2\langle\nabla f(\bx_t,z_t),  \nabla F(\bx_t)\rangle +\|\nabla F(\bx_t)\|^2]\\
    &\le G^2 -\E[\|\nabla F(\bx_t)\|^2]\\
    &\le G^2
\end{align*}
From (\ref{eqn:hessvectorvariance}) we have:
\begin{align*}
    \E[\|\epsilon^H_t\|^2]&\le \E[\|\nabla^2 f(\bx_{t+1}, z_{t+1})(\bx_{t+1} -\bx_t) -  \nabla^2 F(\bx_{t+1})(\bx_{t+1} -\bx_t)\|^2]\\&\le \sigma_H^2\|\bx_{t+1}-\bx_t\|^2
\end{align*}
Also,
\begin{align*}
    \E[\|\hat\epsilon_t\|^2]&\le \E[\|\hat g^{clip}_{t}-\nabla F(\bx_t)\|^2]\\&\le 4G^2
\end{align*}
Finally, also note that we must have:
\begin{align*}
    \|\hat g^{clip}_{t}\|\le G
\end{align*} 
for all $t$ due to our definition of $\hat g^{clip}_{t}$. \\
Let us define another quantity:
\begin{align*}
    \hat \epsilon_{t+1}^{\text{noclip}}=\hat g_{t+1}-\nabla F(\bx_{t+1})
\end{align*}
Now, we derive a recursive formula for $\hat \epsilon_{t+1}^{\text{noclip}}$ in terms of $\hat \epsilon_t$:
\begin{align*}
    \hat \epsilon_{t+1}^{\text{noclip}} &= \hat g_{t+1} - \nabla F(\bx_{t+1})\nonumber\\
    &=(1-\alpha_t)(\hat g^{clip}_{t}+\nabla^2 f(\bx_{t+1}, z_{t+1})(\bx_{t+1} - \bx_t)) +\alpha_t \nabla f(\bx_{t+1}, z_{t+1}) - \nabla F(\bx_{t+1})\nonumber\\
    &=(1-\alpha_t)(\hat \hat g_{tclip} +\nabla^2 f(\bx_{t+1}, z_{t+1})(\bx_{t+1} - \bx_{t})- \nabla F(\bx_{t+1})) + \alpha_t(\nabla f(\bx_{t+1}, z_{t+1}) - \nabla F(\bx_{t+1}))\nonumber\\
    &=(1-\alpha_t)(\hat g^{clip}_{t} - \nabla F(\bx_t)) + (1-\alpha_t)(\nabla^2 f(\bx_{t+1}, z_{t+1})-\nabla^2 F(\bx_{t+1})(\bx_{t+1}-\bx_t))\nonumber\\
&\quad\quad+
(1-\alpha_t)(\nabla F(\bx_t) +  \nabla^2F(\bx_{t+1})(\bx_{t+1}-\bx_t) -\nabla F(\bx_{t+1})) + \alpha_t \epsilon^G_{t+1}
\end{align*}
Now, let's compare $\|\hat\epsilon_{t+1}\|$ and $\|\hat\epsilon_{t+1}^{\text{noclip}}\|$. If $\|\hat g_{t+1}\| \le G$ (no clipping),
\begin{align}
    \hat g_{t+1} &= \hat g_{t+1clip}\nonumber\\
    \Rightarrow\|\hat\epsilon_{t+1}\| &= \|\hat\epsilon_{t+1}^{\text{noclip}}\|\label{eqn:errorrelation}
\end{align}
If $\|\hat g_{t+1}\| > G$, $\|\hat g^{clip}_{t+1}\| = G$. Since $\|\hat g^{clip}_{t+1}\|$ and $\|\hat g_{t+1}\|$ are co-linear, $\|\hat g_{t+1}\| -\|\hat g^{clip}_{t+1}\| = \|\hat g_{t+1} - \hat g^{clip}_{t+1}\|$. Therefore:
\begin{align}
    (\|\hat g_{t+1}\| + \|\hat g^{clip}_{t+1}\|)(\|\hat g_{t+1}\| - \|\hat g^{clip}_{t+1}\|) \ge 2G\|\hat g_{t+1} - \hat g^{clip}_{t+1}\| \label{eqn:firstrelation}
\end{align}
Using (\ref{eqn:gradientbound}) and applying Cauchy-Schwarz inequality, we have:
\begin{align}
    2G\|\hat g_{t+1} - \hat g^{clip}_{t+1}\| &\ge 2\|\nabla F(\bx_{t+1})\|\|\hat g_{t+1} - \hat g^{clip}_{t+1}\| \nonumber \\
    &\ge 2\langle\hat g_{t+1} - \hat g_{t+1clip}, \nabla F(\bx_{t+1})\rangle \label{eqn:secondrelation}
\end{align}
Combining (\ref{eqn:firstrelation}) and (\ref{eqn:secondrelation}):
\begin{align}
(\|\hat g_{t+1}\| + \|\hat g^{clip}_{t+1}\|)(\|\hat g_{t+1}\| - \|\hat g^{clip}_{t+1}\|) &\ge 2\langle\hat g_{t+1} - \hat g^{clip}_{t+1}, \nabla F(\bx_{t+1})\rangle \nonumber \\
      \|\hat g_{t+1}\|^2 - \|\hat g^{clip}_{t+1}\|^2 &\ge 2\langle\hat g_{t+1}, \nabla F(\bx_{t+1})\rangle - 2\langle\hat g^{clip}_{t+1}, \nabla F(\bx_{t+1})\rangle \nonumber\\
      \|\hat g_{t+1}\|^2 - 2\langle\hat g_{t+1}, \nabla F(\bx_{t+1})\rangle &\ge \|\hat g^{clip}_{t+1}\|^2 - 2\langle\hat g^{clip}_{t+1}, \nabla F(\bx_{t+1})\rangle \nonumber \\
      \|\hat g_{t+1}\|^2 - 2\langle\hat g_{t+1}, \nabla F(\bx_{t+1})\rangle + \|\nabla F(\bx_{t+1})\|^2 &\ge \|\hat g^{clip}_{t+1}\|^2 - 2\langle\hat g^{clip}_{t+1}, \nabla F(\bx_{t+1})\rangle+ \|\nabla F(\bx_{t+1})\|^2 \nonumber\\
      \|\hat g_{t+1} - \nabla F(\bx_{t+1})\|^2 &\ge \|\hat g^{clip}_{t+1} - \nabla F(\bx_{t+1})\|^2 \nonumber \\
      \|\hat\epsilon_{t+1}^{\text{noclip}}\|^2 &\ge  \|\hat\epsilon_{t+1}\|^2 \nonumber \\
      \|\hat\epsilon_{t+1}^{\text{noclip}}\| &\ge  \|\hat\epsilon_{t+1}\| \label{eqn:lastine}
\end{align}
From relation (\ref{eqn:errorrelation}) and (\ref{eqn:lastine}):
\begin{align}
     \|\hat\epsilon_{t+1}\| \le \|\hat\epsilon_{t+1}^{\text{noclip}}\| \label{eqn:finaline}
\end{align}
We have:
\begin{align}
    \hat\epsilon_{t+1}^{\text{noclip}} &= (1-\alpha_t)(\hat g^{clip}_{t} - \nabla F(\bx_t)) + (1-\alpha_t)(\nabla^2 f(\bx_{t+1}, z_{t+1})-\nabla^2 F(\bx_{t+1})(\bx_{t+1}-\bx_t))\nonumber\\
&\quad\quad+
(1-\alpha_t)(\nabla F(\bx_t) +  \nabla^2F(\bx_{t+1})(\bx_{t+1}-\bx_t) -\nabla F(\bx_{t+1})) + \alpha_t \epsilon^G_{t+1}\label{eqn:epsrecursion}
\end{align}
Let:
\begin{align}
    \delta_t &= \nabla F(\bx_t) +  \nabla^2F(\bx_{t+1})(\bx_{t+1}-\bx_t) -\nabla F(\bx_{t+1})\nonumber\\
    \|\delta_t\|^2 &\le \frac{\rho^2}{4}\|\bx_{t+1}-\bx_t\|^4 \label{eqn:deltat}
\end{align}
     Equation (\ref{eqn:epsrecursion}) becomes:
\begin{align*}
    \hat \epsilon_{t+1}^{\text{noclip}} &= (1-\alpha_t)\hat{\epsilon_t} +(1-\alpha_t)\epsilon^H_t
    +(1-\alpha_t)\delta_t+\alpha_t \epsilon^G_{t+1}
\end{align*}
Now, remember that we are actually interested in $\E[\|\hat \epsilon_t\|^2]$, so let us take the norm-squared of both sides in the above and use relation (\ref{eqn:finaline}):
\begin{align*}
    \E[\|\hat\epsilon_{t+1}\|^2] &\le (1-\alpha_t)^2\E[\|\hat{\epsilon_{t}}\|^2]+ (1-\alpha_t)^2\E[\|\epsilon^H_{t}\|^2] + (1-\alpha_t)^2\E[\|\delta_t\|^2] + \alpha_t^2\E[||\epsilon^G_{t+1}\|^2] +2(1-\alpha_t)^2\E[\langle\hat{\epsilon_{t}}, \delta_t\rangle] 
\end{align*}
Applying Young's inequality, for any $\lambda$ we have:
\begin{align}
    \langle\hat{\epsilon_{t}}, \delta_t\rangle \leq \frac{\lambda\|\hat{\epsilon_t}\|^2}{2} + \frac{\|\delta_t\|^2}{2\lambda} \label{eqn: crosstermerror}
\end{align}
Using (\ref{eqn:deltat}) and (\ref{eqn: crosstermerror}), we have:
\begin{align*}
    \E[\|\hat\epsilon_{t+1}\|^2]&\le (1-\alpha_t)^2\E[\|\hat{\epsilon_{t}}\|^2] + (1-\alpha_t)^2\sigma_H^2\E[\|\bx_{t+1} - \bx_{t}\|^2] + (1-\alpha_t)^2\frac{\rho^2}{4}\E[\|\bx_{t+1}-\bx_t\|^4] + \alpha_t^2G^2\\
&\quad\quad+ (1-\alpha_t)^2(\lambda \E[\|\hat{\epsilon_t}\|^2] + \frac{\E[\|\delta_t\|^2]}{\lambda} )\\
    &\le(1-\alpha_t)^2\E[\|\hat{\epsilon_{t}}\|^2] + (1-\alpha_t)^2\sigma_H^2\E[\|\bx_{t+1} - \bx_{t}\|^2] + (1-\alpha_t)^2\frac{\rho^2}{4}\E[\|\bx_{t+1}-\bx_t\|^4] + \alpha_t^2G^2 \\
&\quad\quad+ (1-\alpha_t)^2(\lambda \E[\|\hat{\epsilon_t}\|^2]) + (1-\alpha_t)^2 \frac{\rho^2}{4\lambda} \E[\|\bx_{t+1}-\bx_t\|^4] 
\end{align*}
Next, we observe:
\begin{align*}
    \|\bx_t-\bx_{t+1}\|&\le \eta_t \|\hat g^{clip}_{t}\|\\
    &\le \eta_t(\|\nabla F(\bx_t)\| + \|\hat\epsilon_t\|)
\end{align*}
and:
\begin{align*}
    \|\hat g^{clip}_{t}\|^4 \leq G^2\|\hat g^{clip}_{t}\|^2
\end{align*}
So plugging this back in yields:
\begin{align*}
    \E[\|\hat\epsilon_{t+1}\|^2]&\le (1-\alpha_t)^2\E[\|\hat{\epsilon_{t}}\|^2] + (1-\alpha_t)^2\sigma_H^2
\eta_{t}^2\E(\|\nabla F(\bx_t)\|^2 + 2\langle\|\nabla F(\bx_t)\|, \|\hat\epsilon_t\|\rangle+  \|\hat{\epsilon_t}\|^2)\\
&\quad\quad+ \alpha_t^2G^2 +(1-\alpha_t)^2\frac{\rho^2}{4}\eta_{t}^4G^2\E[\|\nabla F(\bx_t)\|^2+ 2\langle\|\nabla F(\bx_t)\|, \|\hat\epsilon_t\|\rangle+  \|\hat{\epsilon_t}\|^2] + (1-\alpha_t)^2(\lambda \E[\|\hat{\epsilon_t}\|^2])\\
&\quad\quad+ (1-\alpha_t)^2 \frac{\rho^2}{4\lambda} \eta_{t}^4G^2\E[]||\nabla F(\bx_t)||^2 + 2\langle\|\nabla F(\bx_t)\|, \|\hat\epsilon_t\|\rangle+  \|\hat{\epsilon_t}||^2] 
\end{align*}
Again applying Young's Inequality with $\lambda = 1$:
\begin{align*}
    \E[\|\hat\epsilon_{t+1}\|^2]&\le (1-\alpha_t)^2\E[\|\hat{\epsilon_{t}}\|^2] + (1-\alpha_t)^2\sigma_H^2
\eta_{t}^2\E[2\|\nabla F(\bx_t)\|^2 + 2\|\hat{\epsilon_t}\|^2] \\
&\quad\quad+(1-\alpha_t)^2\frac{\rho^2}{4}\eta_{t}^4G^2\E[2\|\nabla F(\bx_t)\|^2+ 2 \|\hat{\epsilon_t}\|^2] + \alpha_t^2G^2 \\
&\quad\quad+ (1-\alpha_t)^2(\lambda \E[\|\hat{\epsilon_t}\|^2])+ (1-\alpha_t)^2 \frac{\rho^2}{4\lambda} \eta_{t}^4G^2\E(2\|\nabla F(\bx_t)\|^2+  2\|\hat{\epsilon_t}\|^2)
\end{align*}
Since $ (1-\alpha_t)^2 \le 1 $:
\begin{align*}
     \E[\|\hat\epsilon_{t+1}\|^2]&\le (1-\alpha_t)^2\E[\|\hat{\epsilon_{t}}\|^2] + 2\sigma_H^2
\eta_{t}^2\E[\|\nabla F(\bx_t)\|^2 + \|\hat{\epsilon_t}\|^2] + \alpha_t^2G^2 +\frac{\rho^2}{2}\eta_{t}^4G^2\E[\|\nabla F(x_t)\|^2 \\
&\quad\quad+
\|\hat{\epsilon_t}\|^2] + (\lambda \E[\|\hat{\epsilon_t}\|^2]) +  \frac{\rho^2}{2\lambda} \eta_{t}^4G^2\E[\|\nabla F(\bx_t)\|^2 + \|\hat{\epsilon_t}\|^2]\\
&\le \E[\|\hat{\epsilon_{t}}\|^2]\left[(1-\alpha_t)^2 +2\sigma_H^2\eta_{t}^2 + \frac{\rho^2}{2}\eta_{t}^4G^2 + \lambda + \frac{\rho^2}{2\lambda} \eta_{t}^4G^2 \right]+ \alpha_t^2G^2\\
&+\quad\E[\|\nabla F(\bx_t)\|^2]\left[ \eta_{t}^4 \left(\frac{\rho^2}{2}G^2 + \frac{\rho^2}{2\lambda}G^2\right) + 2\sigma_H^2\eta_{t}^2\right]  
\end{align*}
Now, we will choose parameters in such a way as to ensure:
\begin{align*}
    (1-\alpha_t)^2 +2\sigma_H^2\eta_{t}^2 + \frac{\rho^2}{2}\eta_{t}^4G^2 + \lambda + \frac{\rho^2}{2\lambda} \eta_{t}^4G^2 \le 1 - \frac{5}{12}\alpha_t \ (*) 
\end{align*}
To this end, let $$ \alpha_t \leq 1$$ and $$ \lambda = \frac{\alpha_t}{2}$$
For (*) to be satisfied: 
\begin{align*}
    \eta_{t}^4 \left(\frac{\rho^2}{2}G^2 + \frac{\rho^2}{2\lambda}G^2\right) + 2\sigma_H^2\eta_{t}^2 - \frac{\alpha_t}{12} \le 0 
\end{align*}
Solving the quadratic equation, we get:
\begin{align*}
    \eta_{t}^2 &\le \frac{-2\sigma_H^2 +  \sqrt{4\sigma_H^4+\frac{\rho^2 G^2\alpha_t}{6}+\frac{\rho^2G^2}{3}}}{G^2\rho^2+\frac{2G^2\rho^2}{\alpha_t}}\\
    &\le \frac{-2\sigma_H^2 +  \sqrt{4\sigma_H^4+\frac{\rho^2G^2}{6}+\frac{\rho^2G^2}{3}}}{G^2\rho^2+\frac{2G^2\rho^2}{\alpha_t}}  \\
    &\le \frac{-2\sigma_H^2+\sqrt{4\sigma_H^4+\frac{\rho^2G^2}{2}}}{G^2\rho^2}\frac{\alpha_t}{2}
\end{align*}
Let  $K = \frac{2G^2\rho^2}{-2\sigma_H^2+\sqrt{4\sigma_H^4+\frac{\rho^2G^2}{2}}}$, and suppose that:
\begin{align}
    \eta_{t}^2K \le \alpha_t \label{eqn: momentine}
\end{align}
So overall we get:
\begin{align*}
    \E[\|\hat\epsilon_{t+1}\|^2]&\le\left(1-\frac{5}{12}\alpha_t\right)\E[\|\hat{\epsilon_{t}}\|^2] + \frac{1}{12}\alpha_{t} \E[\|\nabla F(\bx_t)\|^2] + \alpha_t^2G^2\\
    \frac{12\eta_t}{5\alpha_t} \E[\|\hat\epsilon_{t+1}\|^2-\|\hat\epsilon_t\|^2]&\le -\eta_t\E[\|\hat\epsilon_t\|^2] + \frac{\eta_t}{5}\E[\|\nabla F(\bx_t)\|^2] + \frac{12}{5}\eta_t \alpha_t G^2
\end{align*}
Pick $\alpha_t = 2K\eta_t\eta_{t+1}$ and $\eta_t = \frac{1}{Ct^{1/3}}$ ($\frac{\eta_{t}}{\eta_{t+1}} < 2 $ so $\alpha_t$ satisfied (\ref{eqn: momentine})):
\begin{align*}
     \frac{6}{5K\eta_{t+1}} \E[\|\hat{\epsilon_{t+1}}\|^2 -\|\hat{\epsilon_{t}}\|^2 ] \le  -\eta_t \E[|\hat{\epsilon_{t}}\|^2] + \frac{\eta_t}{5} \E[\|\nabla F(\bx_t)\|^2] + \frac{24}{5}\eta_t^3KG^2
\end{align*}
Unfortunately, the coefficient on $\E[\|\hat\epsilon_{t}\|^2]$ above is wrong - it has $\eta_{t+1}$ instead of $\eta_{t}$. Let's correct that:
\begin{align*}
   \frac{6}{5K\eta_{t+1}} \E[\|\hat{\epsilon_{t+1}}\|^2] - \frac{6}{5K\eta_{t}} \E[\|\hat{\epsilon_{t}}\|^2] \le \frac{6}{5K} (\frac{1}{\eta_{t+1}} - \frac{1}{\eta_t})\E[\|\hat{\epsilon_{t}}\|^2] - \eta_t \E[\|\hat{\epsilon_{t}}\|^2] + \frac{\eta_t}{5} \E\|\nabla F(\bx_t)\|^2] + \frac{24}{5}\eta_t^3KG^2
\end{align*}
So, we need to understand $\frac{1}{\eta_{t+1}} - \frac{1}{\eta_t}$:
\begin{align*}
    \frac{1}{\eta_{t+1}} - \frac{1}{\eta_t}&=C((t+1)^{1/3}-t^{1/3})\\
    &\le \frac{C}{3t^{2/3}}\\
    &\le \frac{C^3\eta_t^2}{3}
\end{align*}
Now use Young's Inequality ($ab\le \frac{a^2\lambda}{2}+\frac{b^2}{2\lambda}$) with $a=\sqrt{\eta_t}$ and $b=\eta_t^{3/2}$:
\begin{align*}
    &\le \frac{C^3\lambda \eta_t}{6}+\frac{C^3\eta_t^3}{6\lambda}
\end{align*}
Thus for any $\lambda$ we have:
\begin{align*}
\frac{6}{5K\eta_{t+1}} \E[\|\hat\epsilon_{t+1}\|^2]-\frac{6}{5K\eta_t}\E[\|\hat\epsilon_t\|^2]&\le \frac{6}{5K}\left(\frac{\lambda C^3\eta_t}{6}+\frac{C^3\eta_t^3}{6\lambda}\right)\E[\|\hat\epsilon_{t}\|^2]-\eta_t\E[\|\hat\epsilon_t\|^2] \\
&\quad\quad+ \frac{\eta_t}{5}\E[\|\nabla F(\bx_t)\|^2] + \frac{24}{5}KG^2\eta_t^3\\
&=-\eta_t\left(1 - \frac{C^3\lambda}{5K}\right)\E[\|\hat\epsilon_t\|^2] + \frac{\eta_t}{5}\E[\|\nabla F(\bx_t)\|^2]  \\
&\quad\quad+ \eta_t^3\left(\frac{24}{5}KG^2+\frac{C^3E[\|\hat{\epsilon_{t}}\|^2]}{5K\lambda}\right)
\end{align*}
So, let us set $\lambda = \frac{5K}{4C^3}$ and use $\E[\|\hat\epsilon_t\|^2]\le4 G^2$:
\begin{align*}
    \frac{6}{5K\eta_{t+1}} \E[\|\hat\epsilon_{t+1}\|^2]-\frac{6}{5K\eta_t}\E[\|\hat\epsilon_t\|^2]&\le-\frac{3\eta_t}{4}\E[\|\hat\epsilon_t\|^2] + \frac{\eta_t}{5}\E[\|\nabla F(\bx_t)\|^2]  \\
&\quad\quad+ \eta_t^3\left(\frac{24}{5}KG^2+\frac{16C^6G^2}{25K^2}\right)
\end{align*}
\end{proof}
\subsection{Proof of Lemma \ref{normgradient}}
\label{sec:proof}
\normgradient*
\begin{proof}
Assuming (\ref{eqn:Lsmooth}) holds, with $\bx = \bx_t$ and $\delta = \bx_{t+1} - \bx_t = \eta \hat g_t$, we have:
\begin{align}
    F(\bx_{t+1}) &\le F(\bx_t) - \eta\langle\nabla F(\bx_t), \frac{\hat g_t}{\|\hat g_t \|}\rangle + \frac{L\eta^2}{2} \label{eqn:smoothine}
\end{align}
Let us analyze the inner product term via some case-work: Suppose $\|\hat \epsilon_t\| \le \frac{1}{2}\|\nabla F(\bx_t)\|$. Then we have $ \|\nabla F(\bx_t) + \hat \epsilon_t \| \le \frac{3}{2}\|\nabla F(\bx_t)\|$ so that:
\begin{align*}
    -\langle\nabla F(\bx_t), \frac{\hat g_t}{\|\hat g_t \|}\rangle &= -\langle\nabla F(\bx_t), \frac{\nabla F(\bx_t) + \hat \epsilon_t}{\|\nabla F(\bx_t) + \hat \epsilon_t \|}\rangle \\
    &\le \frac{-\|\nabla F(\bx_t)\|^2}{\|\nabla F(\bx_t) + \hat \epsilon_t \|} + \frac{\|\nabla F(\bx_t)\|\|\hat \epsilon_t\|}{\|\nabla F(\bx_t) + \hat \epsilon_t \|} \\
    &\le -\frac{2}{3}\|\nabla F(\bx_t)\| + 2\|\hat \epsilon_t\|
\end{align*}
On the other hand, if $\|\hat \epsilon_t\| > \frac{1}{2}\|\nabla F(\bx_t)\|$, then we have:
\begin{align*}
    -\langle\nabla F(\bx_t), \frac{\hat g_t}{\|\hat g_t \|}\rangle &\le 0 \\
    &\le -\frac{2}{3}\|\nabla F(\bx_t)\|  + \frac{2}{3}\|\nabla F(\bx_t)\|\\
    &\le -\frac{2}{3}\|\nabla F(\bx_t)\|  + \frac{4}{3}\|\hat \epsilon_t\|
\end{align*}
So either way, we have $ -\langle\nabla F(\bx_t), \frac{\hat g_t}{\|\hat g_t \|}\rangle  \le -\frac{2}{3}\|\nabla F(\bx_t)\|  + 2\|\hat \epsilon_t\| $. Now sum (\ref{eqn:smoothine}) over t and rearrange to obtain:
\begin{align*}
    \E[\|\nabla F(\bx_t)\|] &\le \frac{3(F(\bx_1) - F(\bx_{T+1}))}{2\eta} + \frac{3L\eta T}{4} + 3\sum_{t=1}^{T} \|\hat\epsilon_t\| \\
    &\le \frac{3\Delta}{2\eta} + \frac{3L\eta T}{4} + 3\sum_{t=1}^{T} \|\hat\epsilon_t\|
\end{align*}
Finally, observe that since $\bx_t$ is chosen uniformly at random from $\bx_1,...,\bx_T$, we have $\E\|\nabla F(\bx_t)\| = \frac{1}{T}\sum_{t=1}^{T}\|\nabla F(\bx_t)\|$ to conclude the results.
\end{proof}
\subsection{Proof of Theorem \ref{normalizedhess}}
\label{sec:proof5}
\normalizedhess*
\begin{proof}
Let us write a recursive expression for $\hat \epsilon_t$:
\begin{align*}
    \hat \epsilon_t &= \hat g_t - \nabla F(\bx_t) \\
    &= (1-\alpha) (\hat g_{t-1} + \nabla^2 f(\bx_t, z_t)(\bx_t - \bx_{t-1}))+ \alpha\nabla f(\bx_t, z_t) - \nabla F(\bx_t)\\
    &= (1-\alpha) (\hat g_{t-1} + \nabla^2 f(\bx_t, z_t)(\bx_t - \bx_{t-1}) -\nabla F(\bx_t) )+ \alpha(\nabla f(\bx_t, z_t) - \nabla F(\bx_t))
\end{align*}
Let us define $v_t = \nabla f(\bx_t, z_t) - \nabla F(\bx_t)$ and $w_t = \frac{\nabla^2 f(\bx_t, z_t)(\bx_t - \bx_{t-1}) - \nabla^2 F(\bx_t)(\bx_t - \bx_{t-1})}{\|\bx_t - \bx_{t-1}\|}$. Note that:
\begin{align*}
    w_t &= \frac{\nabla^2 f(\bx_t, z_t)(\bx_t - \bx_{t-1}) - \nabla^2 F(\bx_t)(\bx_t - \bx_{t-1})}{\|\bx_t - \bx_{t-1}\|} \\
    &= \frac{\nabla^2 f(\bx_t, z_t)(\bx_t - \bx_{t-1}) - \nabla^2 F(\bx_t)(\bx_t - \bx_{t-1})}{\eta}
\end{align*}
Finally, define $\delta_t =\nabla F(\bx_{t-1}) + \nabla^2 F(\bx_t)(\bx_t - \bx_{t-1}) -\nabla F(\bx_t) $. Since F is $\rho$-second-order smooth, we must have $\|\delta_t\| \le \frac{\rho}{2}\|\bx_t - \bx_{t-1}\|^2 = \frac{\rho}{2}\eta^2$. Now we write:
\begin{align*}
    \hat \epsilon_t&= (1-\alpha) [\hat g_{t-1} + \nabla^2 f(\bx_t, z_t)(\bx_t -\bx_{t-1}) -\nabla F(\bx_t)]+ \alpha(\nabla f(\bx_t, z_t) - \nabla F(\bx_t))\\
    &= (1-\alpha)[\hat g_{t-1} - \nabla F(\bx_{t-1}) + \|\bx_t-\bx_{t-1}\|w_t + \nabla F(\bx_{t-1}) + \nabla^2 F(\bx_t, z_t)(\bx_t - \bx_{t-1})-\nabla F(\bx_t)] + \alpha v_t \\
    &= (1-\alpha)(\hat \epsilon_{t-1} + \|\bx_t-\bx_{t-1}\|w_t + \delta_t)+\alpha v_t
\end{align*}
Now unroll this recursive expression:
\begin{align*}
    \hat \epsilon_t &= (1-\alpha)^{t-1}\hat \epsilon_1 + \sum_{\tau=0}^{t-1}(1-\alpha)^{\tau+1} (\|\bx_t-\bx_{t-1}\|w_{t-\tau} + \delta_{t-\tau}) + \alpha (1-\alpha)^{\tau}y_{t-\tau}
\intertext{Observe that $\hat\epsilon_1 = v_1$ and apply triangle inequality:}
\hat \epsilon_t &\le (1-\alpha)^{t-1}\|v_1\| +  \frac{\eta^2\rho}{2}\sum_{\tau=0}^{t-1}(1-\alpha)^{\tau+1} + \eta\|\sum_{\tau=0}^{t-1}(1-\alpha)^{\tau+1}w_{t-\tau}\| + \alpha\|\sum_{\tau=0}^{t-1}(1-\alpha)^{\tau}v_{t-\tau}\| 
\intertext{Take expectation of the expression:}
    \E[\|\hat\epsilon_t\|] &\le (1-\alpha)^{t-1}\sigma_G + \frac{\eta^2\rho}{2}\sum_{\tau=0}^{t-1}(1-\alpha)^{\tau+1}+ \eta\sigma_H\sqrt{\sum_{\tau=0}^{t-1}(1-\alpha)^{2\tau+2}} +\sigma_G\alpha\sqrt{\sum_{\tau=0}^{t-1}(1-\alpha)^{2\tau}}
\intertext{All the sums can be upper bounded by $\sum_{\tau=0}^{\infty} (1-\alpha)^{\tau} = \frac{1}{\alpha}$:}
    &\le (1-\alpha)^{t-1}\sigma_G + \frac{\eta^2\rho}{2\alpha} + \frac{\eta\sigma_G}{\sqrt{\alpha}} + \sigma_G\sqrt{\alpha}
\end{align*}
Next, sum over t:
\begin{align*}
    \sum_{t=1}^{T} \E[\|\hat\epsilon_t\|] &\le \frac{\sigma_G}{\alpha} +\frac{\eta^2\rho T}{2\alpha} + \frac{\eta\sigma_H T}{\sqrt{\alpha}} + \sigma_G\sqrt{\alpha}T
\intertext{Applying Lemma \ref{normgradient}:}
    \E[\|\nabla F(\bx_t)\|] &\le \frac{3\Delta}{2\eta T} + \frac{3L\eta}{4} +\frac{3}{T}\sum_{t=1}^{T} \|\hat\epsilon_t\| \\
    &\le \frac{3\Delta}{2\eta T} + \frac{3L\eta}{4} + \frac{3\sigma_G}{T\alpha} +\frac{3\eta^2\rho}{2\alpha} + \frac{3\eta\sigma_H}{\sqrt{\alpha}} + 3\sigma_G\sqrt{\alpha}
\end{align*}
Now, with $\alpha = \min \{ \max \{ \frac{1}{T^{2/3}}, \frac{\Delta^{4/5}\rho^{2/5}}{T^{4/5}\sigma_G^{6/5}}, \frac{(2\Delta\sigma_H)^{2/3}}{T^{2/3}\sigma_G^{4/3}} \}, 1 \}$ and $\eta = \min \{ \frac{\sqrt{2\Delta}\alpha^{1/4}}{\sqrt{T(L\sqrt{\alpha}+4\sigma_H})}, \frac{(\Delta\alpha)^{1/3}}{(\rho T)^{1/3}} \}$, use Lemma \ref{normalizedbound} in the appendix to finish the proof.
\end{proof}
\subsection{Proof of Theorem \ref{thm:adaptive}}
\label{sec:proof6}
\adaptive*
\begin{proof}
Define the potential:
\begin{align*}
    \Phi_t = F(\bx_t) +\frac{6}{5K\eta_t }\|\hat\epsilon_t\|^2
\end{align*}
Then:
\begin{align*}
    \E[\Phi_{t+1}-\Phi_t]&= \E\left[F(\bx_{t+1})-F(\bx_t) + \frac{6}{5K\eta_{t+1} }\|\hat\epsilon_{t+1}\|^2-\frac{6}{5K\eta_t }\|\hat\epsilon_t\|^2\right]
    \intertext{Applying Lemma \ref{errorLemma} and Lemma \ref{thm:onestep} then sum over t:}
    \E[\Phi_{T+1}-\Phi_1] &\le \sum_{t=1}^T\E\left[-\frac{\eta_t}{4}\|\nabla F(\bx_t)\|^2 + \frac{3\eta_t}{4}\|\hat\epsilon_t\|^2 + \frac{6}{5K\eta_{t+1} }\|\hat\epsilon_{t+1}\|^2-\frac{6}{5K\eta_t }\|\hat\epsilon_t\|^2 \right]\\
    &\le \E\left[\sum_{t=1}^T-\frac{\eta_t}{20}\|\nabla F(\bx_t)\|^2 +\frac{24}{5}Kc^2\ln(T+1)+  \frac{16G^4}{25K^2c^3} \ln T \right]
\end{align*}
Reordering the term:
\begin{align*}
    \E\left[\sum_{t=1}^T \eta_t\|\nabla F(\bx_t)\|^2 \right] &\le \E\left[20(\Phi_1-\Phi_{T+1}) + 96Kc^2\ln(T+1)+  \frac{64G^4}{5K^2c^3} \ln T \right]
\end{align*}
Also:
\begin{align*}
    \E[\Phi_1-\Phi_{T+1}] &= \E[F(\bx_1)-F(\bx_{T+1}) + \frac{6}{5K\eta_1}\|\hat\epsilon_1\|^2 - \frac{6}{5K\eta_{T+1}}\|\hat\epsilon_{T+1}\|^2] \\
    &\le \Delta + \frac{6\sigma_G^2w^{1/3}}{5Kc}
\end{align*}
Plug in:
\begin{align*}
    \E\left[\sum_{t=1}^T \eta_t\|\nabla F(\bx_t)\|^2 \right] &\le 20\left(\Delta + \frac{6\sigma_G^2w^{1/3}}{5Kc}\right) +96Kc^2\ln(T+1)+  \frac{64G^4}{5K^2c^3} \ln T
\end{align*}
Now, let us relate $\E\left[\sum_{t=1}^T \eta_t\|\nabla F(\bx_t)\|^2 \right]$ to $\E\left[\sum_{t=1}^T\|\nabla F(\bx_t)\|^2 \right]$. Since $\eta_t$ is decreasing:
\begin{align*}
    \E\left[\sum_{t=1}^T \eta_t\|\nabla F(\bx_t)\|^2 \right] \ge \E\left[\eta_T \sum_{t=1}^T \|\nabla F(\bx_t)\|^2 \right]
\end{align*}
From Cauchy-Schwartz:
\begin{align*}
    \E[1/\eta_T]\E\left[\eta_T\sum_{t=1}^T \|\nabla F(\bx_t)\|^2\right] \ge  \E\left[\sqrt{\sum_{t=1}^T \|\nabla F(\bx_t)\|^2}\right]^2
\end{align*}
Let $M = \frac{1}{c}\left(20(\Delta + \frac{6\sigma_G^2w^{1/3}}{5Kc}) + 96Kc^2\ln(T+1)+  \frac{64G^4}{5K^2c^3} \ln T\right)$, then:
\begin{align*}
    \E\left[\sqrt{\sum_{t=1}^T \|\nabla F(\bx_t)\|^2}\right]^2 &\le \E\left[\frac{cM}{\eta_T}\right]
    \\ &\le \E\left[M(w+\sum_{t=1}^TG_t^2)^{1/3}\right]
\end{align*}
Let $\zeta_t = \nabla f(\bx_t, z) - \nabla F(\bx_t)$ so that $\E[\|\zeta_t\|^2] \le \sigma_G^2$. Then we have $G_t^2 = \|\nabla F(\bx_t) + \zeta_t\|^2 \le 2\|\nabla F(\bx_t)\|^2 + 2\|\zeta_t\|^2 $. Thus:
\begin{align*}
    \E\left[\sqrt{\sum_{t=1}^T \|\nabla F(\bx_t)\|^2}\right]^2 &\le \E\left[M\left(w+ 2\sum_{t=1}^T\|\zeta_t\|^2)\right)^{1/3} + 2^{1/3}M\left(\sum_{t=1}^T\|\nabla F(\bx_t)\|^2\right)^{1/3}\right]
    \\ &\le M(w+2T\sigma_G^2)^{1/3} + \E\left[2^{1/3}M\left(\sqrt{\sum_{t=1}^T\|\nabla F(\bx_t)\|^2}\right)^{2/3}\right]  \\
    &\le M(w+2T\sigma_G^2)^{1/3} +2^{1/3}M\left(\E\left[\sqrt{\sum_{t=1}^T \|\nabla F(\bx_t)\|^2}\right]\right)^{2/3}
\end{align*}
Define $X = \sqrt{\sum_{t=1}^T\|\nabla F(\bx_t)\|^2}$. The the above can be rewritten as:
\begin{align*}
    (\E[X])^2 \le M(w+2T\sigma_G^2)^{1/3} + 2^{1/3}M (\E[X])^{2/3}
\end{align*}
This implies that either $(\E[X])^2 \le 2M(w+2T\sigma_G^2)^{1/3}$ or $(\E[X])^2 \le 2\times2^{1/3}M (\E[X])^{2/3}$. Solving for $\E[X]$ in these two cases, we get:
\begin{align*}
    \E[X] \le \sqrt{2M}(w+2T\sigma_G^2)^{1/6} +2M^{3/4}
\end{align*}
Finally, by Cauchy-Schwartz we have $\sum_{t=1}^T\|\nabla F(\bx_t)\|/T \le X/ \sqrt{T}$. Therefore:
\begin{align*}
    \E\left[\sum_{t=1}^T\frac{\|\nabla F(\bx_t)\|}{T}\right] \le \frac{w^{1/6}\sqrt{2M}+2M^{3/4}}{\sqrt{T}} +\frac{2\sigma_G^{1/3}}{T^{1/3}}
\end{align*}
with $M = \frac{1}{c}\left(20(\Delta + \frac{6\sigma_G^2w^{1/3}}{5Kc}) +96Kc^2\ln(T+1)+  \frac{64G^4}{5K^2c^3} \ln T\right)$
\end{proof}
\end{document}